%% file: main.tex
\documentclass{article}

\PassOptionsToPackage{square, numbers}{natbib}

\usepackage[preprint]{neurips_2025}

\usepackage[utf8]{inputenc} 
\usepackage[T1]{fontenc}    
\usepackage[colorlinks=true, allcolors=blue]{hyperref}
\usepackage{url}            
\usepackage{booktabs}       
\usepackage{amsfonts}       
\usepackage{nicefrac}       
\usepackage{microtype}      
\usepackage[table]{xcolor}         
\usepackage{amsmath}
\usepackage{amsthm}
\usepackage{newtxmath}
\usepackage{multirow}
\usepackage{booktabs}
\usepackage{wrapfig}
\usepackage{enumitem}
\usepackage{subcaption}

\usepackage{listings}
\usepackage{xcolor} 
\definecolor{palered}{gray}{0.95}
\definecolor{lightgray}{gray}{0.6}

\usepackage{cleveref}

\makeatletter
\crefname{section}{\S\@gobble}{\S\@gobble}
\crefname{subsection}{\S\@gobble}{\S\@gobble}
\crefname{proposition}{Prop.}{Props.}
\crefname{figure}{Fig.}{Figs.}
\crefname{table}{Table}{Tables}
\makeatother

\usepackage[textsize=tiny]{todonotes}

\newcommand{\std}[1]{\text{\scriptsize{$\pm #1$}}}

\newtheorem{theorem}{Theorem}[section]

\title{Adaptive Destruction Processes for Diffusion Samplers}

\author{%
 Timofei Gritsaev\thanks{Correspondence to \texttt{tgritsaev@constructor.university}.}\\
 Constructor University\\HSE University
 \And
 Nikita Morozov\\
 HSE University
 \And
 Kirill Tamogashev\\
 University of Edinburgh
 \And
 Daniil Tiapkin\\
 CMAP, CNRS, \'Ecole polytechnique\\
 LMO, Universit\'e Paris-Saclay
 \And
 Sergey Samsonov\\
 HSE University
 \And
 Alexey Naumov\\
 HSE University
 \And
 Dmitry Vetrov\\
 Constructor University
 \And
 Nikolay Malkin\\
 University of Edinburgh
}

\newcommand{\R}{\mathbb{R}} 
\newcommand{\E}{\mathbb{E}}
\newcommand{\dt}{\Delta t}

\newcommand{\dd}{\mathrm{d}}

\newcommand{\eg}{\emph{e.g.}}
\newcommand{\ie}{\emph{i.e.}}

\makeatletter

\renewcommand{\section}{%
  \@startsection{section}{1}{\z@}%
                {-1.5ex \@plus -0.2ex \@minus -0.2ex}%
                { 1.0ex \@plus  0.2ex \@minus  0.2ex}%
                {\large\bf\raggedright}%
}
\renewcommand{\subsection}{%
  \@startsection{subsection}{2}{\z@}%
                {-1.5ex \@plus -0.2ex \@minus -0.2ex}%
                { 0.25ex \@plus  0.2ex}%
                {\normalsize\bf\raggedright}%
}
\renewcommand{\paragraph}{%
  \@startsection{paragraph}{4}{\z@}%
                {0.ex}%
                {-1em}%
                {\normalsize\bf}%
}

\makeatother

\begin{document}

\maketitle

\begin{abstract}

    \looseness=-1
    This paper explores the challenges and benefits of a trainable destruction process in diffusion samplers -- diffusion-based generative models trained to sample an unnormalised density without access to data samples. Contrary to the majority of work that views diffusion samplers as approximations to an underlying continuous-time model, we view diffusion models as discrete-time policies trained to produce samples in very few generation steps. We propose to trade some of the elegance of the underlying theory for flexibility in the definition of the generative and destruction policies. In particular, we decouple the generation and destruction variances, enabling both transition kernels to be learned as unconstrained Gaussian densities. We show that, when the number of steps is limited, training both generation and destruction processes results in faster convergence and improved sampling quality on various benchmarks. Through a robust ablation study, we investigate the design choices necessary to facilitate stable training. Finally, we show the scalability of our approach through experiments on GAN latent space sampling for conditional image generation.
\end{abstract}

\section{Introduction}

Probabilistic inference is concerned with sampling from and estimating statistics of distributions defined by an unnormalised density.
Given an energy function $\mathcal{E}:\R^d \to \R$, we want to sample from:
\begin{equation}
    p_{\rm target}(x) = \frac{e^{-\mathcal{E}(x)}}{Z}; \quad Z = \int_{\R^d} e^{-\mathcal{E}(x)}\,\dd x
\end{equation}
and estimate the normalising constant $Z$.

Monte Carlo methods like AIS \cite{neal1998annealedimportancesampling} or Hamiltonian MCMC \cite{neal2011mcmc,hoffman2014no} operate by evolving a family of particles using a tractable proposal kernel, which in the limit (of infinite time or number of particles) converge to the target distribution.
However, such methods may require many sampling steps to converge or need many particles to make the bias sufficiently small.
In addition, some of them rely on access to $\nabla_x\mathcal{E}(x)$, which may not always be available. Some methods use simple distributions with adaptive parameters \cite{bugallo2017adaptive} or neural models \cite{samsonov2022local,midgley2022flow} as proposals within Monte Carlo methods. This allows to increase the convergence speed of the chain and improve the sampling quality. However, fully amortised algorithms such as Boltzmann generators \cite{noe2019boltzmann} are trained to directly produce samples from the target distribution without a Monte Carlo outer loop.

The successes of diffusion models as distribution approximators \cite{sohl2015diffusion, ho2020denoising, song2021scorebased} provide motivation for applying diffusion-based techniques for fully amortised sampling. Building on these successes, diffusion samplers (see \Cref{sec:diffusion_samplers}) were proposed as a more scalable alternative to classic MCMC algorithms. Unlike diffusion models, they do not assume access to samples from the target density and are suitable for cases when the target distribution is represented by a black-box energy function. 

There are many algorithms for training diffusion samplers (see \Cref{sec:diffusion_samplers} and related work in \Cref{apx:related}). One approach, first explored in \cite{zhang2022path,vargas2023denoising}, is based on the theory of stochastic processes and views diffusion samplers as approximators to continuous-time stochastic dynamics. Another approach replaces continuous dynamics with a time-discrete Markov chain and trains the diffusion sampler as a discrete policy using reinforcement learning objectives \cite{lahlou2023theory}.
It has been shown that in the continuous-time limit the two approaches coincide \cite{berner2025discrete}: discrete diffusion samplers approximate continuous ones well if the number of discretisation steps is sufficiently large. 
In practice, though, it is desirable to make the number of steps as small as possible to make sampling faster. 
However, a continuous-time diffusion sampler discretised too coarsely will no longer model the target distribution accurately. 

The assumption of underlying continuous-time dynamics necessitates constraints on the generative and destruction processes: for example, their diffusion coefficients must coincide (see \Cref{apx:incompar_proc}). Surprisingly, discrete-time diffusion samplers have more flexibility: both generation and destruction transition kernels can be flexibly learned, for example, as Gaussians with mean and variance parametrised by neural networks. In this paper, we take advantage of this flexibility to train both generation and destruction processes, obtaining faster and more accurate samplers. 
We study different objectives for the generation and destruction processes and explore possible techniques for stabilising training. In summary, this paper makes the following findings:
\begin{enumerate}[left=0pt,nosep,label=(\arabic*)]
   \item Training both the generation and destruction processes of diffusion samplers significantly improves sampling quality and normalising constant estimation across various benchmarks, particularly when the number of sampling steps is small or the energy landscape has narrow modes.
   \item Sampling highly benefits from the variance of the generation process being learned and decoupled from that of the destruction process -- a unique feature of the discrete-time formulation.
   \item Successful joint training of both processes in diffusion samplers hinges on appropriate model architecture and selection of training objectives. We point out practical guidelines for stabilising the joint training of generation and destruction processes -- such as shared backbones, separate optimisers, target networks, and prioritised replay buffers -- and validate them through a comprehensive ablation study.
   \item Finally, our approach is scalable to higher-dimensional problems, as shown in  experiments on text-conditional sampling in the latent space of a pretrained StyleGAN3~\cite{karras2021alias}, demonstrating quantitative and qualitative benefits of training the destruction process.
\end{enumerate}

\section{Background}

\subsection{Diffusion samplers in continuous time}
\label{sec:diffusion_samplers}

We consider the problem of sampling from a distribution over $\mathbb{R}^d$ with density
$p_{\text{target}}(x) = e^{-\mathcal{E}(x)} / Z; $ where $ Z = \int e^{-\mathcal{E}(x)}\,\dd x$ is a (usually intractable) normalising constant and 
$ \mathcal{E} : \mathbb{R}^d \to \mathbb{R}$ is an energy function.\footnote{For absolutely continuous densities, we abuse notation and use the distribution and density interchangeably. Further, all claims about SDEs hold under basic regularity conditions, \eg, those in \cite[][\S B.1]{berner2025discrete}.} We have access only to $\mathcal{E}(x)$. 
We want to sample from $p_{\text{target}}$ by simulating the dynamics of a process characterised by a (forward or generative) It\^o SDE:
\begin{equation}    
    \dd \overrightarrow{X_t} = \overrightarrow{f_\theta}(\overrightarrow{X_t}, t)\,\dd t + g(t)\,\dd\overrightarrow{W_t},\quad \overrightarrow{X_0}\sim p_0,
\label{eq:fsde}
\end{equation}
where $p_0$ is a tractable distribution (\eg, Gaussian or Dirac).
The drift function $\overrightarrow{f_\theta}$ is a neural network with parameters $\theta$, and we want to find $\theta$ such that $\overrightarrow{X_1}\sim p_{\rm target}$, \ie, the terminal samples $\overrightarrow{X_1}$ are distributed as $p_{\rm target}$. 
There may be many SDEs of the form in~\eqref{eq:fsde} stochastically transporting $p_0$ to $p_{\rm target}$.
To set a learning target for $\theta$, we additionally assume a backward (or destructive) SDE (fixed or also trainable):
\begin{equation}    
    \dd \overleftarrow{X_t} = \overleftarrow{f_\varphi}(\overleftarrow{X_t}, t)\,\dd t + g(t)\,\dd \overleftarrow{W_t},\quad \overleftarrow{X_1}\sim p_{\rm target}.
\label{eq:bsde}
\end{equation}
Given a fixed process~\eqref{eq:bsde}, there is a unique forward SDE of the form~\eqref{eq:fsde} that defines the same process, \ie, the processes $\overrightarrow{X_t}$ and $\overleftarrow{X_t}$ coincide. 
In particular, the marginal distributions at time 1 would then coincide, so $\overrightarrow{X_1}\sim p_{\rm target}$.
Thus, training objectives aim to enforce \emph{time reversal} -- to minimise some divergence between the generation and destruction processes with respect to their parameters.

One approach to solving this problem, related to continuous-time stochastic control, involves writing a divergence functional defined in continuous time. This functional may compute a discrepancy between path space measures (\ie, distributions over trajectories, see \cite{nusken2021solving,vargas2022bayesian,zhang2022path,vargas2024transport,berner2024optimal}), which can then be approximately minimised in a time discretisation. Alternatively, one can use time-local objectives involving the processes' marginal distributions at intermediate times $t$, derived from the Fokker-Planck-Kolmogorov or Hamilton-Jacobi-Bellman PDEs (see \cite{nusken2023interpolating,mate2023learning, sun2024physicsinformed}).

Another approach is to immediately replace the generation and destruction SDEs by corresponding discrete-time Markov chains and to learn the dynamics given a fixed number of steps. The discretisation may be defined by Euler-Maruyama integration of the SDEs (see \cref{sec:training_discrete_time}). One can then borrow objectives from reinforcement learning to enforce time reversal of discrete-time processes.
This approach was notably explored in \cite{lahlou2023theory,sendera2024improved} and is in theory more general, as it allows more flexibility in the choice of transition kernels \cite{phillips2024metagfn,volokhova2024towards}.

The two approaches become equivalent in the continuous-time limit: as the time discretisation becomes finer, objectives for training discrete-time transition kernels approach their continuous-time counterparts, justifying the use of varying time discretisations for training \cite{berner2025discrete}.
However, learning in continuous time presents a notable limitation: it necessitates that variances of both generation and destruction
processes remain the same. Otherwise, divergence functionals cannot be defined, as we demonstrate in~\Cref{apx:incompar_proc}.
In \Cref{sec:relaxation}, we will study how the assumption of fixed variances can be relaxed in the discrete-time approach and how this can be beneficial for modelling.

\subsection{Diffusion samplers in discrete time}
\label{sec:training_discrete_time}

In order to learn the time-discrete variant of~\eqref{eq:fsde} we use a $T$-step Euler-Maruyama scheme to obtain a Markov chain: $X_0 \rightarrow X_{\Delta t} \rightarrow \ldots \rightarrow X_1;\; \Delta t = 1 / T$.\footnote{For notational clarity we assume a uniform discretisation, but all derivations hold for variable time steps.} The discrete dynamics can be written as:
\begin{equation}
    X_{t + \dt} = 
    X_{t} + f_{\theta}(X_t, t)\,\dt + g(t) \sqrt{\dt} \, \xi_t, \quad X_0 \sim p_0(X_0),\; \xi_t \sim \mathcal{N}(0, I_d) 
    \label{eq:transition}
\end{equation}
with the conditional probability density for each timestep being
\begin{equation}
    \overrightarrow{p_\theta}(X_{t + \dt} \mid X_t) = \mathcal{N} ( 
        X_{t + \Delta t} ; X_{t} + f_{\theta}(X_t, t)\,\dt, \: g(t)^2 \dt \,I_d
    )
    \label{eq:transtion_density}
\end{equation}
We can now write the distribution over generation trajectories in the following way:
\begin{equation}
    \overrightarrow{p_\theta}(X_0, X_{\Delta t}, \ldots, X_1) = p_0(X_0)\overrightarrow{p_\theta}(X_{\dt, \ldots, 1} \mid X_0)
    = p_0(X_0) \prod_{i=1}^T \overrightarrow{p_\theta}(X_{i\dt} \mid X_{(i-1)\dt}) 
    \label{eq:trajectory_density_f}
\end{equation}
We can similarly\footnote{The reverse-time discretisation can use the reverse Euler-Maruyama scheme, or (as in \cite{zhang2022path} and other works, including ours) the reversal of a forward-time Euler-Maruyama scheme, which may not coincide -- in particular, this yields the multiplier $\frac{t}{t+\Delta t}$ on the variance in \eqref{eq:transtion_density_b_fixed}. The two coincide as $\Delta t\to0$; see, \eg, \cite[][Prop.\ 3.5]{berner2025discrete}.} discretise the destruction SDE~\eqref{eq:bsde} to obtain a time-reversed Markov chain with transition kernel $\overleftarrow{p_\varphi}$:
\begin{equation}
    \overleftarrow{p_\varphi}(X_0, X_{\Delta t}, \ldots, X_1) = 
    p_{\text{target}}(X_1)\overleftarrow{p_\varphi}(X_{0 , \ldots, (T - 1)\dt} \mid X_1)
    = p_{\text{target}}(X_1) \prod_{i=1}^T \overleftarrow{p_\varphi}(X_{(i-1)\dt} \mid X_{i\dt}),
\end{equation}
where if $p_0$ is Dirac the transition $\overleftarrow{p_\varphi}(X_0\mid X_{\Delta t})$ is understood to be Dirac as well (and both are taken to have density 1 in the density ratios representing Radon-Nikodym derivatives in \cref{sec:objectives}). When both generation and destruction kernels are trained we expect the following equality to hold: 
\begin{equation}
    p_0(X_0)\overrightarrow{p_\theta}(X_{\dt, \ldots, 1} \mid X_0) =
    p_{\text{target}}(X_1)\overleftarrow{p_\varphi}(X_{0 , \ldots, (T - 1)\dt} \mid X_1)
    \label{eq:tb} 
\end{equation}
or $p_0 \overrightarrow{p_\theta} = p_{{\rm target}} \overleftarrow{p_\varphi}$ for brevity.

The majority of past work (see \Cref{apx:related})
focuses on learning the generation process given a fixed destruction process, although some \cite{richter2024improved} derive objectives for optimising both processes. In this work we propose to exploit the aforementioned flexibility of time-discrete formulation of diffusion samplers to parametrise both drift and variance of generation and destruction processes using neural networks. We further detail our approach in \Cref{sec:objectives} and show experimental findings in \Cref{sec:experiments}.

\section{Related work}
\label{apx:related}

{
\paragraph{Diffusion samplers.} While diffusion models were first introduced with the goal of approximating a distribution from which samples are available by an iterative denoising process~\cite{sohl2015diffusion,ho2020denoising,song2021scorebased}, \emph{diffusion samplers} -- on which the modern line of work started with \cite{vargas2022bayesian,zhang2022path,vargas2023denoising} and continues to rapidly expand \cite[][\emph{inter alia}]{vargas2024transport,richter2024improved,akhoundsadegh2024iterated,sendera2024improved,blessing2025underdamped,havens2025adjoint} -- seek to amortise the cost of sampling from a given target density $p_{\text{target}}$, which can be queried for the energy and possibly its gradient, by training a generation process to approximate it. Algorithms such as PIS~\cite{zhang2022path} and DDS~\cite{vargas2023denoising} achieve this by minimising the KL divergence between destruction and generation processes. Other divergences have been considered for their better numerical properties, for example, \cite{richter2024improved} uses a second-moment divergence~\cite{richter2020vargrad}. In \cite{lahlou2023theory,zhang2022unifying}, diffusion samplers are understood as a generalised instance of GFlowNets \cite{bengio2021flow}, which are deep reinforcement learning algorithms that allow for using exploration techniques and off-policy training and do not require access to the gradients of $p_{\text{target}}$ (although physics-informed parametrisations may use this gradient). Subsequent work \cite{zhang2024diffusion, sendera2024improved, kim2025adaptive} has applied training objectives from GFlowNet literature~\cite{malkin2022trajectory, madan2023learning,pan2023better} to the diffusion sampling case and explored the benefits of the flexible off-policy training that GFlowNets allow \cite{sendera2024improved,phillips2024metagfn,kim2025adaptive}. Finally, \cite{berner2025discrete} builds theoretical connections among various diffusion sampling objectives and their continuous-time limits.

\paragraph{GFlowNets.} Generative flow networks~\cite[GFlowNets;][]{bengio2021flow, bengio2023gflownet} have been introduced as a general framework for sampling from distributions by solving a sequential decision-making problem, training a stochastic policy to construct objects step by step. GFlowNets represent a synthesis of variational inference and reinforcement learning (RL) paradigms~\cite{malkin2023gflownets, zimmermann2023a, tiapkin2024generative, deleu2024discrete}. Beyond their original application in biological structure discovery \cite[\eg,][]{jain2022biological,shen2024tacogfn,cretu2025synflownet}, GFlowNets have found uses in probabilistic inference over symbolic latent variables \cite{deleu2022bayesian,van2023nesi,hu2023gfnem,deleu2023jsp,zhou2024phylogfn,silva2024streaming}, combinatorial optimisation \cite{zhang2023robust,zhang2023let,kim2025ant}, and reasoning or planning in language \cite{hu2024amortizing,song2024latent,lee2025learning,younsi2025accurate}.
Although initially defined for discrete spaces, the framework can be extended to continuous domains~\cite{lahlou2023theory} and non-acyclic state spaces (\ie, nonconstrutive actions) \cite{brunswic2024theory,morozov2025revisiting}. GFlowNets approach the training of amortised samplers from a RL perspective and correspondingly require solving the RL challenges of exploration~\cite{rector2023thompson, sendera2024improved, phillips2024metagfn, madan2025towards, kim2025adaptive}, credit assignment~\cite{malkin2022trajectory, pan2023better, jang2024learning}, and generalisation \cite{silva2025when,atanackovic2025investigating}. Relevant to our work, design and training of the destruction process in \emph{discrete} problems is a focus of \cite{malkin2023gflownets,mohammadpour2024maximum,jang2024pessimistic,gritsaev2025optimizing}, where learning the destruction process is shown to improve convergence and mode discovery.

\looseness=-1
\paragraph{Learning the destruction process in diffusion.} A recent line of work~\cite{bartosh2024neuraldiffusion, bartosh2024neuralflow, sahoo2024diffusion, nielsen2024diffenc} studies learning the destruction process (called the `forward' process, clashing with the opposite convention in place for diffusion samplers) for diffusion models in image domains, instead of using a fixed one, such as VP or VE SDE~\cite{song2021scorebased}. 
This results in a broader class of models and is shown to improve log-likelihood and visual quality of the generated samples, as well as reducing the required number of training steps. GFlowNet objectives were applied to this problem in \cite{lahlou2023theory,zhang2022unifying}. Learning the destruction process for diffusion models was also studied in discrete domains, \eg, \cite{kong2023autoregressive} learns the order in which nodes are removed from a graph by the destruction process. 

Our work is also adjacent to the \textbf{Schr\"odinger bridge (SB) problem}. In continuous time, the learning of destruction and generation processes that stochastically transport a source distribution $p_0$ to $p_{\rm target}$ is known as a \emph{bridge problem}. The SB problem seeks the unique bridge that is closest to some reference process and can be solved by algorithms based on iterative proportional fitting in a time discretisation \cite[IPF;][]{vargas2021solving,debortoli2021diffusion} or various other methods \cite{chen2022likelihood,stromme2023sampling,shi2023dsbm,tong2024simulation}. Unlike diffusion samplers, these algorithms typically assume samples from $p_{\rm target}$ are available. Many of these algorithms, when typical approximations are made in training, at convergence yield bridges that are not necessarily the SB. In this work, we constrain the destruction process to be a bridge from $p_{\rm target}$ to $p_0$, meaning that in the continuous-time limit iterative proportional fitting would converge in a single step. However, the discrete-time TLM update for the destruction process (\Cref{sec:objectives}), which trains $\overleftarrow{p_\varphi}$ on trajectories sampled from $\overrightarrow{p_\theta}$, is in fact equivalent to an (unconverged) maximum-likelihood IPF step as in \cite{vargas2021solving}.
}

\section{Methodology}

\subsection{Relaxing constraints on transitions}
\label{sec:relaxation}

We first propose to relax the constraint that the generation and destruction processes arise from a pair of SDEs \eqref{eq:fsde} and \eqref{eq:bsde} with fixed time-dependent variance. We begin with the destruction process that is the reversal of time-discretised Brownian motion with fixed diffusion coefficient $\sigma$ starting from $p_0=\delta_0$, as considered by \cite{zhang2022path} and many later works. For this process, the transitions have the form:
\begin{align}
    \overrightarrow{p_\theta}(X_{t + \dt} \mid X_t) &= \mathcal{N} ( 
        X_{t + \Delta t} ; X_{t} + f_{\theta}(X_t, t)\,\dt, \: \sigma^2 \dt \,I_d
    ),
    \label{eq:transtion_density_f_fixed}
    \\
    \overleftarrow{p}(X_{t} \mid X_{t + \dt}) &= \mathcal{N} \left( 
        X_{t} ; \frac{t}{t+\dt} X_{t + \dt} , \: \frac{t}{t+\dt}\sigma^2 \dt \,I_d
    \right),
    \label{eq:transtion_density_b_fixed}
\end{align}
where $\overleftarrow{p}(x_0\mid x_{\dt})$ is understood as Dirac $\delta_0(x_0)$.

We relax this constraint and allow both generation and destruction processes to be Gaussian with arbitrary means and variances, while maintaining the requirement that the destruction process leads to $p_0$ at the last step. This is done by learning the generation and destruction means and variances as corrections to those in \eqref{eq:transtion_density_f_fixed} and \eqref{eq:transtion_density_b_fixed}, respectively. By introducing corrections we get the following generative transition kernel:
\begin{align}
    \overrightarrow{p_\theta}(X_{t + \dt} \mid X_t) &= \mathcal{N} ( 
        X_{t + \Delta t} ; X_{t} + f_{\theta}(X_t, t)\,\dt, \: \text{diag}(\textcolor{blue}{\gamma_\theta(X_t, t)})\sigma^2 \dt \,I_d
    ), \nonumber \\
    \gamma_\theta(X_t, t) &= \exp\left\{C_1 \tanh\left({\text{NN}^{(1)}_\theta(X_t, t)}\right)\right\}
    \label{eq:transtion_density_f_learnable}
\end{align}
and the following destruction kernel:
\begin{align}
    \overleftarrow{p_\varphi}(X_{t} \mid X_{t + \dt}) &= \mathcal{N} \left( 
        X_{t} ; \text{diag}(\textcolor{blue}{\alpha_{\varphi}(X_t, t)}) \, \frac{t}{t+\dt} X_{t + \dt} , \: \text{diag}(\textcolor{blue}{\beta_{\varphi}(X_t, t)})\frac{t}{t+\dt}\sigma^2 \dt \,I_d
    \right), \nonumber \\
    \alpha_{\varphi}(X_t, t) &= 1 + C_2 \tanh\left(\text{NN}^{(2)}_\varphi(X_t, t)\right), \;\; 
    \beta_{\varphi}(X_t, t) = 1 + C_2 \tanh\left(\text{NN}^{(3)}_\varphi(X_t, t)\right),
    \label{eq:transtion_density_b_learnable}
\end{align}
where the $\text{NN}^{(i)}$ represent neural networks with $d$-dimensional outputs and $C_1, C_2$ are constants. Such a parametrisation guarantees that the generation process variance remains between $\frac{1}{C_1}$ and $C_1$ times that of the baseline one in \eqref{eq:transtion_density_f_fixed}, while the multiplicative deviations of the destruction process mean and variance from the basic ones in \eqref{eq:transtion_density_b_fixed} are bounded between $1 - C_2$ and $1 + C_2$. We empirically found such a combination to work better than using the same parametrisation for both generation and destruction processes. This parametrisation does not impose any other constraints on the generation and destruction processes -- they are no longer discretisations of SDEs defining absolutely continuous path space measures.

\looseness=-1
In theory, the transitions could be more general than Gaussian, as long as their density is tractable to sample and differentiate. The flexibility of this approach allows easy generalisation to other distributions, such as mixtures of Gaussians, transition kernels on manifolds, or even non-parametric distributions.

\paragraph{Why train the variance of the generation process?} When the number of time discretisation steps is large, the reverse of a fixed destruction process $\overleftarrow{p_\varphi}$ is well approximated by Gaussian transitions with the same variances. However, when the number of steps is small, the discrete-time generation process should be able to adapt to the discretisation of the destruction process. A preliminary experiment in~\cite{sendera2024improved} showed that allowing for the variance of the generation process to be learned when the destruction process is fixed can improve sampling quality. Thus, we consider the combination of a fixed destruction process and generation process with learnable mean and variance as one of the baselines in our experiments.

\paragraph{Why train the destruction process?} The ability to train the destruction process increases the degrees of freedom in the feasible space of solutions $(p_0\overrightarrow{p_\theta}, p_{\rm target} \overleftarrow{p_\varphi})$ when minimising the discrepancy in the time-reversal condition~\eqref{eq:tb}. Modulo optimisation errors, the solution with a trainable destruction process is guaranteed to be at least as good as the one with a fixed destruction process: the latter is a special case of the former, where the learned correction to the destruction process is zero.

\subsection{Training objectives for generation and destruction processes}
\label{sec:objectives} 

Learning discrete Markov process requires optimising a divergence between two distributions $\mathbb{D}(p_0\overrightarrow{p_\theta}\|p_{{\rm target}} \overleftarrow{p_\varphi})$ that enforces equality in \eqref{eq:tb}. One option can be a reverse KL divergence: \begingroup
\allowdisplaybreaks
\begin{align}
    \mathbb{D}_{\rm KL}(p_0\overrightarrow{p_\theta}\|p_{{\rm target}} \overleftarrow{p_\varphi}) &= {\rm KL} \left(
        p_0(X_0)\overrightarrow{p_\theta}(X_{\Delta t, \ldots, 1} \mid X_0) \|
        p_{\text{target}}(X_1)\overleftarrow{p_\varphi}(X_{0 , \ldots,(T - 1)\dt} \mid X_1) 
    \right) \nonumber \\
    &= \mathbb{E}_{X_{0, \Delta t,\ldots, 1} \sim \overrightarrow{p_\theta}(X_{0,\Delta t, \ldots, 1} )} \left[ 
        \log\frac{p_0(X_0)\overrightarrow{p_\theta}(X_{\Delta t, \ldots, 1} \mid X_0)}
        {\exp(-\mathcal{E}(X_1))\overleftarrow{p_\varphi}(X_{0 , \ldots,(T - 1)\dt} \mid X_1)}
    \right]+\log Z.
\label{eq:rev_kl}
\end{align}
\endgroup

Here the log-normalising constant $\log Z$ does not affect the optimisation. Methods that propose to optimise this objective include \cite{zhang2022path, vargas2023denoising}. They use the reparametrisation trick to rewrite~\eqref{eq:rev_kl} as an expectation over the noises used in integration (the $\xi_t$ in~\eqref{eq:transition}). Although this scheme yields an unbiased estimator, it requires backpropagating though the simulation of the generation process. Alternatively, one can use second-moment divergences of the form
\begin{equation}
    \!\!\!\!\!\!\mathbb{D}_{\tilde p}(p_0\overrightarrow{p_\theta}\|p_{{\rm target}} \overleftarrow{p_\varphi})
    =
    \E_{X_{0,\Delta t, \ldots, 1} \sim \tilde{p}(X_{0,\Delta t, \ldots, 1} )} \left[ \log\dfrac{p_0(X_0)\overrightarrow{p_\theta}(X_{\Delta t, \ldots, 1} \mid X_0)}{\exp(-\mathcal{E}(X_1))\overleftarrow{p_\varphi}(X_{0 , \ldots,(T - 1)\dt} \mid X_1) }+\log\hat{Z}\right]^2
    \label{eq:second_moment}
\end{equation}
where $\tilde p$ is some full-support proposal distribution not necessarily equal to $\overrightarrow{p_\theta}$ and $\log\hat{Z}$ is a scalar. This scalar absorbs the (unknown) true normalising constant $Z$ of $p_{\rm target}$, allowing to use the unnormalised log-density $-\mathcal{E}(X_1)$ in place of $\log p_{\rm target}(X_1)=-\mathcal{E}(X_1)-\log Z$ in the denominator. The scalar $\log\hat Z$ can be a learned parameter (yielding the loss called \textbf{trajectory balance} \cite[\textbf{TB};][]{malkin2022trajectory}) or analytically computed to minimise the loss on each batch of training trajectories (yielding the log-variance or \textbf{VarGrad} loss \cite{richter2020vargrad}). In either case, when the loss is minimised to 0, $\log\hat Z$ becomes equal to the true log-partition function $\log Z$.

Unlike the KL divergence~\eqref{eq:rev_kl}, which is an expectation over a distribution depending on $\theta$, second-moment divergences avoid backpropagating through sampling, allowing the use of off-policy exploration techniques to construct the proposal distribution $\tilde p$. On the other hand, when $\tilde p=p_\theta$, the expected gradient of $\mathbb{D}_{\tilde p}$ with respect to $\theta$ -- \emph{but not with respect to $\varphi$} -- remarkably coincides with that of $\mathbb{D}_{\rm KL}$ \cite{malkin2023gflownets}.

The divergence in~\eqref{eq:rev_kl} can be used for optimisation of $\overleftarrow{p_{\varphi}}$ in addition to $\overrightarrow{p_\theta}$. Notice that this is equivalent to log-likelihood maximisation with respect to $\varphi$, since the gradient of~\eqref{eq:rev_kl} takes the form:
\begin{align}
    \nabla_\varphi\mathbb{D}_{\rm KL}(p_0\overrightarrow{p_\theta}\|p_{{\rm target}} \overleftarrow{p_\varphi}) 
    &= \nabla_\varphi\mathbb{E}_{X_{0, \ldots, 1} \sim \overrightarrow{p_\theta}(X_{0, \ldots, 1} )}  \left[ -
        \log \left( \overleftarrow{p_\varphi}(X_{0 , \ldots,(T - 1)\dt} \mid X_1)\right)
    \right] \\
    &= 
    \nabla_\varphi\mathbb{E}_{X_{0, \ldots, 1} \sim \overrightarrow{p_\theta}(X_{0, \ldots, 1} )} \left[ 
        \sum_{i = 1}^T -\log \overleftarrow{p_\varphi} (X_{(i-1)\dt} \mid X_{i\dt}) \
    \right]
\label{eq:tlm}
\end{align}
In the more general GFlowNet setting, such approach for learning the destruction process is called \textbf{trajectory likelihood maximisation} \cite[\textbf{TLM};][]{gritsaev2025optimizing}.

\begin{wraptable}[7]{r}{0.5\linewidth}
\vspace*{-1.25em}
\caption{Summary of the objectives studied for learning the generation and destruction processes.}\vspace*{-0.75em}
\label{tab:objectives}
\centering
\resizebox{\linewidth}{!}{
\begin{tabular}{@{}lcc}
\toprule
 & \multicolumn{2}{c}{$\overrightarrow{p_\theta}$ (generation)} \\
 \cmidrule(lr){2-3}
$\overleftarrow{p_\varphi}$ (destruction)  $\downarrow$ &  $2^{\text{nd}}$ moment           & Reverse KL \\
\midrule
\textit{Fixed} & \cite{richter2024improved,sendera2024improved} & PIS \cite{zhang2022path}, DDS \cite{vargas2023denoising} \\ 
\midrule
$2^{\text{nd}}$ moment         & TB$_{\theta,\varphi}$ & PIS$_\theta$ + VarGrad$_\varphi$ \\
Rev. KL & TB$_\theta$ + TLM$_\varphi$     &  PIS$_\theta$ + TLM$_\varphi$ \\
\bottomrule
\end{tabular}
}
\end{wraptable}

\Cref{tab:objectives} summarises the possible combinations of objectives for destruction and generation policies. In \Cref{sec:experiments} we numerically study the behaviour of the aforementioned objectives. 

\subsection{Techniques for stable joint optimisation}
\label{sec:techniques}

\looseness=-1
We found that simultaneously training generation and destruction processes requires careful design. In this section we review the utilised techniques that facilitate stable and effective training. Detailed descriptions are presented in \Cref{apx:stability_techniques}, and a numerical study is carried out in \Cref{sec:ablation_results}.

\paragraph{Shared backbone.}
It is possible to parametrise generation and destruction processes with separate neural networks consisting of a MLP backbone and linear heads for mean and log-variance or to use a shared backbone with different heads for each process. We motivate the latter parametrisation by the idea that sharing representations between the two processes facilitates faster convergence. This design choice was also explored and motivated in GFlowNet literature~\cite{malkin2022trajectory, gritsaev2025optimizing}.

\paragraph{Separate optimisers.} When representations are shared, the backbone parameters receive gradient signals from both the losses used for training $\theta$ and $\varphi$. Their gradients may have different scales, especially when different training objectives are employed for them, which may cause instability if both gradient updates use the same adaptive-momentum optimiser. Thus, we use separate optimisers for the generation and destruction process losses.

\paragraph{Learning rates.} 
Since the destruction process sets a learning target for the generation process, making it learnable causes this target to become nonstationary, leading to instabilities in optimisation (although the joint optimisation has no saddle point at the global optimum). Thus, tuning relative learning rates is critical for stable training. The same effect was studied in~\cite{gritsaev2025optimizing} for learnable destruction processes in GFlowNets.

\paragraph{Target network.}
Target networks are a well-known RL stabilisation technique \cite{mnih2015human} that replaces regression targets in a loss function by variants smoothed over training time (by a lagging update, \eg, using exponential moving average~\cite{silver2014deterministic}). We use a target network for $\varphi$ when computing the loss gradient for $\theta$ and vice versa.

\paragraph{Prioritised experience replay.}
Replay buffers are another common RL technique that allows to prevent forgetting by resampling previously used trajectories. We use prioritised experience replay \cite[{PER};][]{schaul2016prioritized} with prioritisation by the loss, which was also utilised for GFlowNet training by~\cite{tiapkin2024generative}. The replay ratio is a hyperparameter that determines the number of training trajectories sampled from the replay buffer per one newly sampled trajectory.

\paragraph{Better exploration in off-policy methods.}
Training with second-moment divergences allows a flexible choice of behavior policy $\tilde p$ (see \eqref{eq:second_moment}). We use the existing techniques proposed by \cite{lahlou2023theory,sendera2024improved} to facilitate exploration during training, namely, increasing the variance of the generation process for sampling trajectories for training, as well as local search based on Langevin dynamics. 

\subsection{Connection to reinforcement learning}
\label{sec:rl_connection}

In this section we briefly discuss a fact that provides important motivation and context for some of the techniques mentioned in \Cref{sec:techniques}: is possible to formulate the training of a diffusion sampler in discrete time as an RL problem with entropy regularisation (also called soft RL \cite{neu2017unified, geist2019theory}). Although, up to our knowledge, it was never directly stated in this form in previous literature on diffusion samplers, this fact is well-understood, and a number of works on diffusion samplers frame their training as a stochastic control task in continuous time~\cite{zhang2022path,vargas2024transport,berner2024optimal}. The equivalence with soft RL in discrete time can be directly shown by combining the analysis of~\cite{tiapkin2024generative, deleu2024discrete}, which show GFlowNet training is  equivalent to a soft RL task, and~\cite{lahlou2023theory, sendera2024improved, berner2025discrete}, which show that discrete-time diffusion samplers are a special case of continuous GFlowNets. We provide the details and the proof in \Cref{apx:soft_rl}. 

This connection gives an important motivation for using such techniques as target networks~\cite{mnih2015human}, replay buffers~\cite{schaul2016prioritized}, and various exploration methods, which are all well-studied in the RL literature, for the training of diffusion samplers. Finally, we mention that \cite{gritsaev2025optimizing} analyses trainable destruction processes in GFlowNets from the RL perspective and theoretically and experimentally stipulates the use of techniques for enforcing stability in the training of the destruction process, such as controlled learning rates and target networks.

\section{Experiments}
\label{sec:experiments}

\subsection{Benchmarking setup on synthetic tasks}
\label{sec:synthetic_tasks}

\paragraph{Energies.} We consider several target densities which are commonly used in the diffusion samplers literature, as well as more challenging variants of these distributions (see \Cref{apx:energies} for details):
\begin{itemize}[left=0pt,nosep]
    \item \textbf{Gaussian mixture models (GMM).} In addition to the simple mixture of 25 regularly spaced Gaussians (\textbf{25GMM}) and 40 irregularly spaced Gaussians (\textbf{40GMM}) in $\R^2$, we also consider a more difficult $\R^3$ variant (\textbf{125GMM}) and a mixture of 25 Gaussians with anisotropic covariance matrices (\textbf{(Slightly) Distorted 25GMM}). 
    \item \textbf{Funnel} \cite{neal2003slice}. We consider two 10-dimensional funnel distributions from past work. Two versions (termed \textbf{Easy/Hard})  have been considered in the literature due to an error in the definition originating in \cite{zhang2022path} (see \cite[][\S B.1]{sendera2024improved}); we study both.
    \item \textbf{Manywell} \cite{noe2019boltzmann}. This is a 32-dimensional distribution with $2^{16}$ modes that are separated by energy barriers, making it challenging for diffusion samplers to avoid mode collapse. Just as for GMMs, we also define a \textbf{Distorted} variant that does not factor into 16 independent 2D distributions.
\end{itemize}

\paragraph{Metrics.} Following the literature, we primarily use ELBO and EUBO metrics to evaluate the quality of trained samplers~\cite{blessing2024beyond}. The two metrics lower- and upper-bound the log-partition function of the target distribution, respectively, and EUBO requires access to ground truth samples, which can be obtained exactly for GMMs and Funnel and approximately for Manywell. (See \Cref{sec:metrics} for definitions.) We also compute the 2-Wasserstein distance between generated and ground truth samples as a more direct measure of sample quality. 

\paragraph{Algorithms compared.} We use two objectives for optimisation of the generation process: trajectory balance (TB$_\theta, $~\cite{malkin2022trajectory}) and reverse KL (equivalent to PIS~\cite{zhang2022path} and so denoted PIS$_\theta$). We compare four settings with each generation process objective: 
\begin{itemize}[left=0pt,nosep]
    \item Fixed generation variances, fixed destruction process (as in most prior work);
    \item Learnable generation variances, fixed destruction process;
    \item Learnable generation variances, destruction process learned using reverse KL (TLM$_\varphi$) or TB (TB$_\varphi$). (When the generation process is learned using reverse KL, we replace TB$_\varphi$ by VarGrad$_\varphi$, since the log-normalising constant is not learned.)
\end{itemize}

\paragraph{Training.} We train all samplers with a varying number of time discretisation steps $T$. The same number of steps is used for training and for sampling. We use uniform time discretisation for all environments except mixtures of Gaussians, which use a harmonic discretisation scheme~\cite{berner2025discrete}. Notably, unlike some past work, we do \emph{not} use the physics-informed Langevin parametrisation from \cite{zhang2022path} in the generation process drift, so as to reduce training cost and study the effects of our modifications in isolation. All details can be found in~\Cref{sec:training_details}.

\subsection{Results on synthetic tasks}
\label{sec:main_results}

\begin{figure}[t]
    \vspace*{-1em}
    \begin{tabular}{@{}c@{\hspace{0.01\textwidth}}c@{\hspace{0.01\textwidth}}c@{\hspace{0.01\textwidth}}c@{}}
    25GMM
    &
    Distorted 25GMM
    &
    \multicolumn{2}{c}{40GMM}
    \\
    \cmidrule(lr){1-1}
    \cmidrule(lr){2-2}
    \cmidrule(lr){3-4}
    \includegraphics[width=0.243\textwidth,trim=0 0 40 46,clip]{
        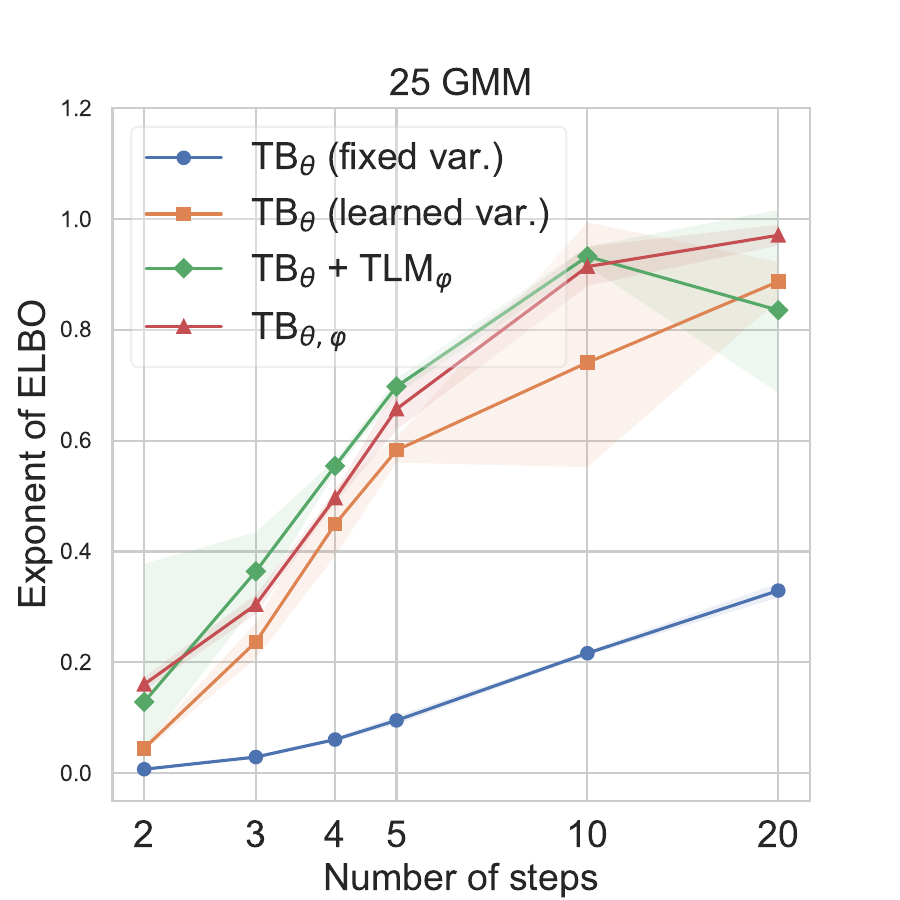}
    &
    \includegraphics[width=0.243\textwidth,trim=0 0 40 46,clip]{
        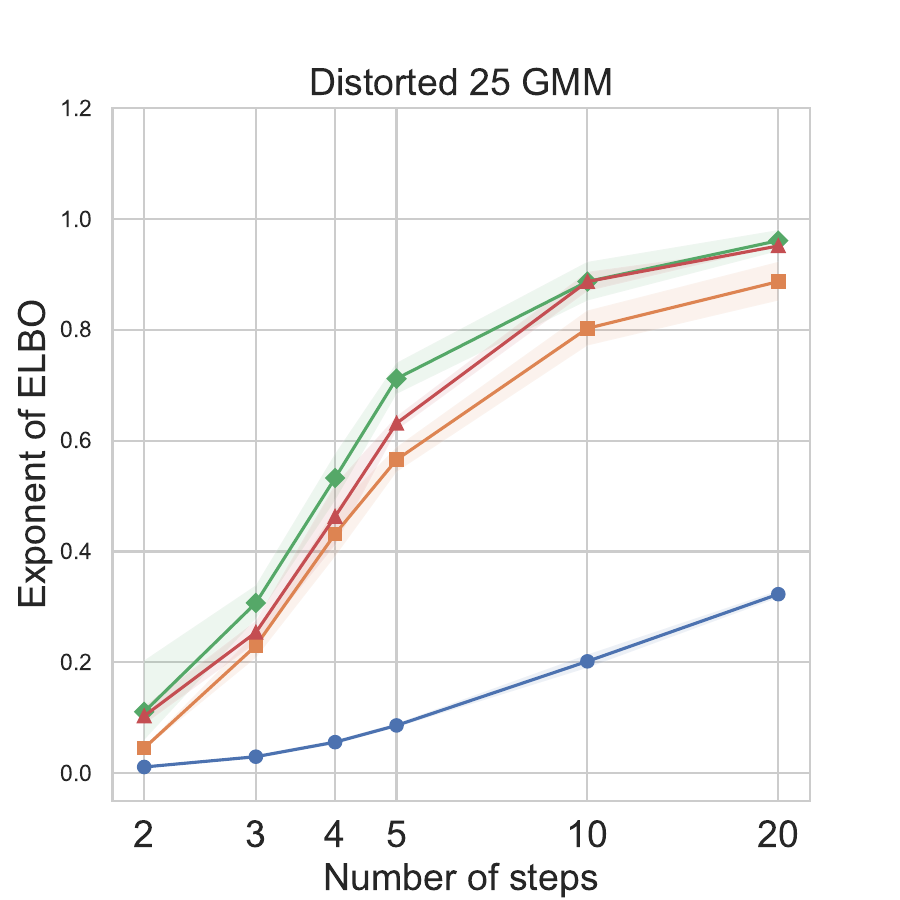}
    &
    \includegraphics[width=0.243\textwidth,trim=0 0 40 46,clip]{
        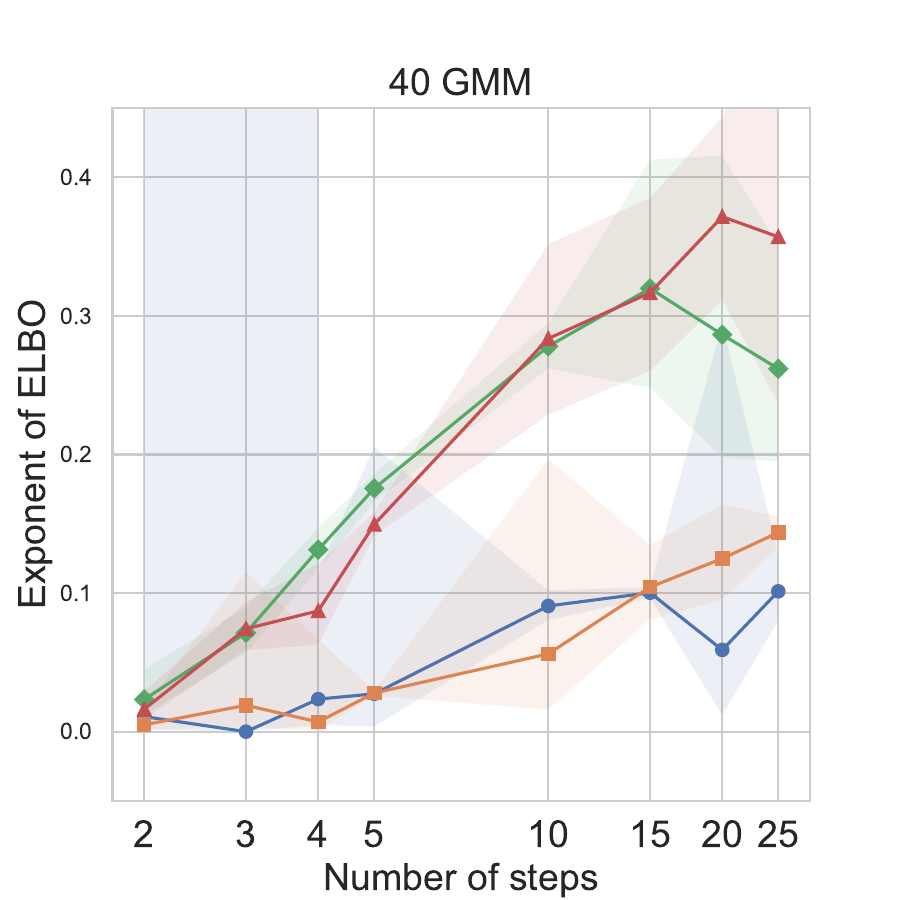}
    &
    \includegraphics[width=0.243\textwidth,trim=0 0 40 46,clip]{
        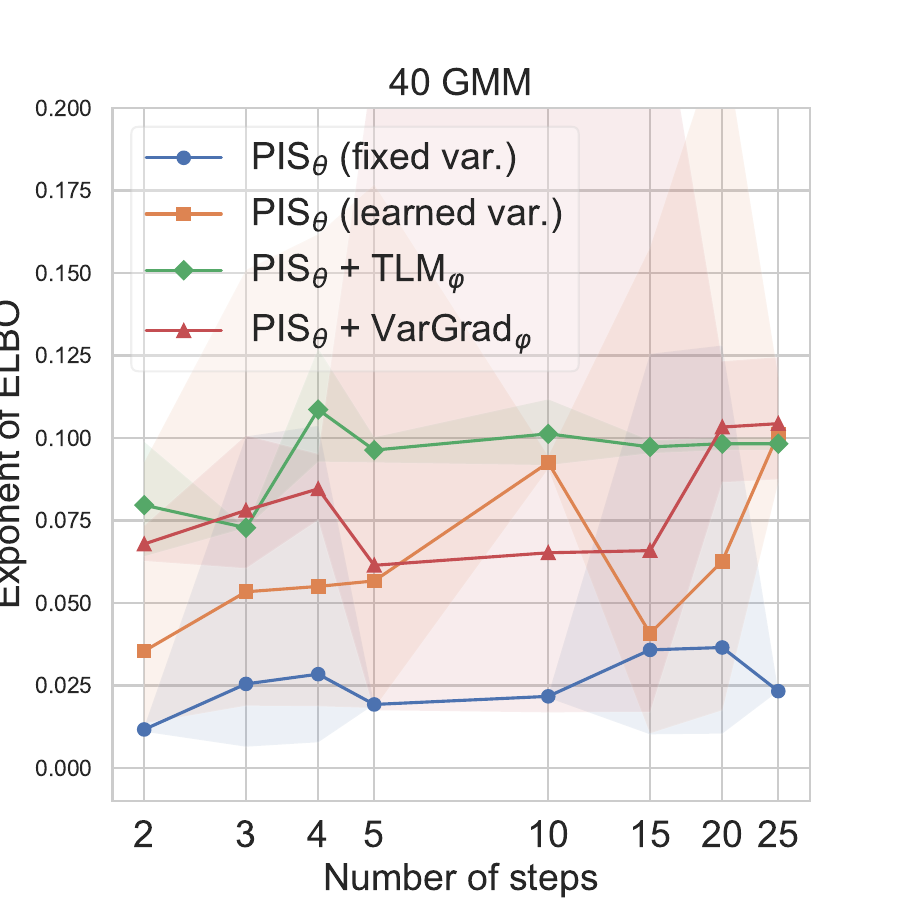}
    \\
    \cmidrule(lr){1-3}
    \cmidrule(lr){4-4}
    \multicolumn{3}{c}{TB gen.\ loss}
    &
    Rev.\ KL gen.\ loss
    \end{tabular}
    \caption{
        We compare performance of $\rm TB_{\theta}$ (fixed var.), $\rm TB_{\theta}$ (learned var.), 
        $\rm TB_{\theta} + TLM_{\varphi}$, $\rm TB_{\theta, \varphi}$ on three GMM targets. The rightmost plot shows PIS-like generation process objectives on \textbf{40GMM}. Results are compared using ELBO. Mean and std over 3 seeds, collapsed runs excluded.
    }

    \label{fig:main_plots}
    \vspace{0.1cm}
    {
    \centering
    \begin{tabular}{@{}c@{\hspace{0.01\textwidth}}c@{\hspace{0.01\textwidth}}c@{\hspace{0.01\textwidth}}c@{}}
    125GMM 
    &
    Easy Funnel
    &
    Hard Funnel
    &
    Distorted Manywell \\
    \includegraphics[width=0.243\textwidth,trim=0 0 40 49,clip]{
        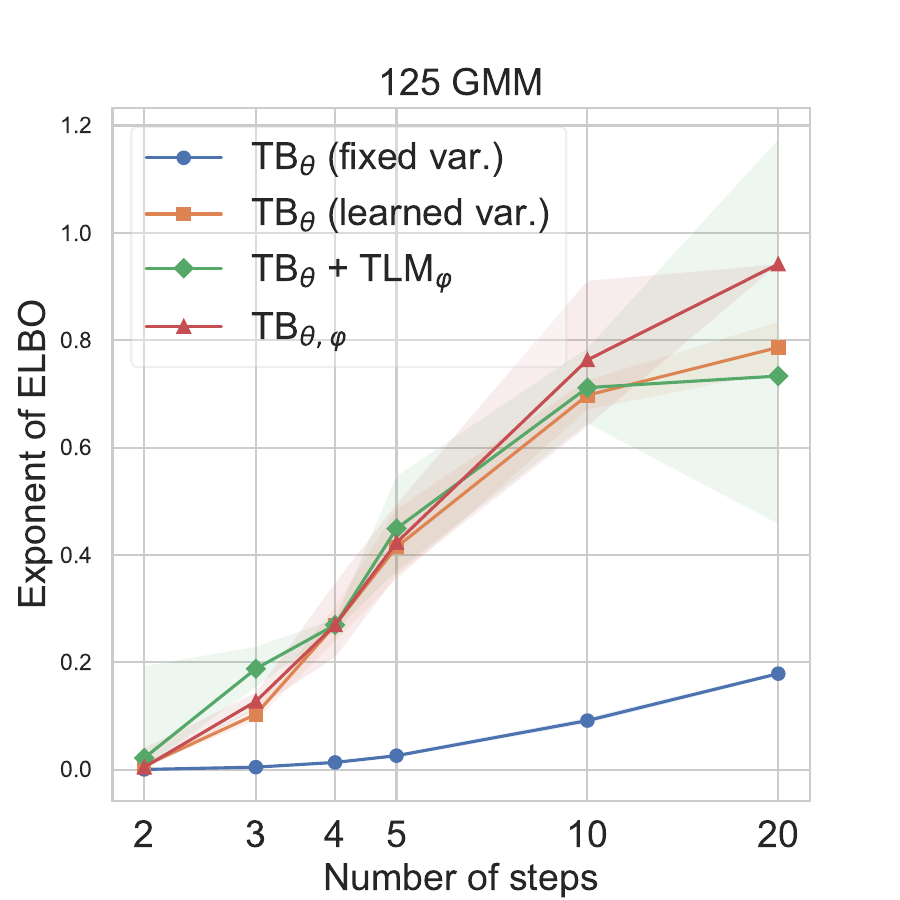
    } 
    &\includegraphics[width=0.243\textwidth,trim=0 0 40 49,clip]{
        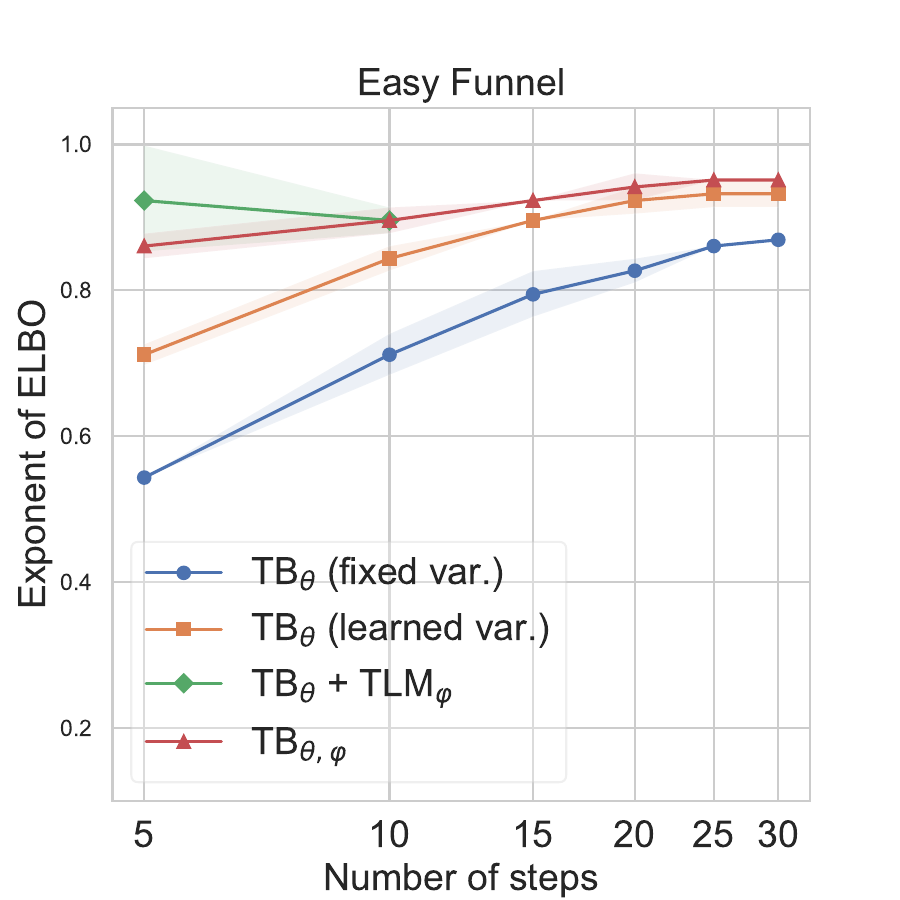
    } 
    &
    \includegraphics[width=0.243\textwidth,trim=0 0 40 49,clip]{
        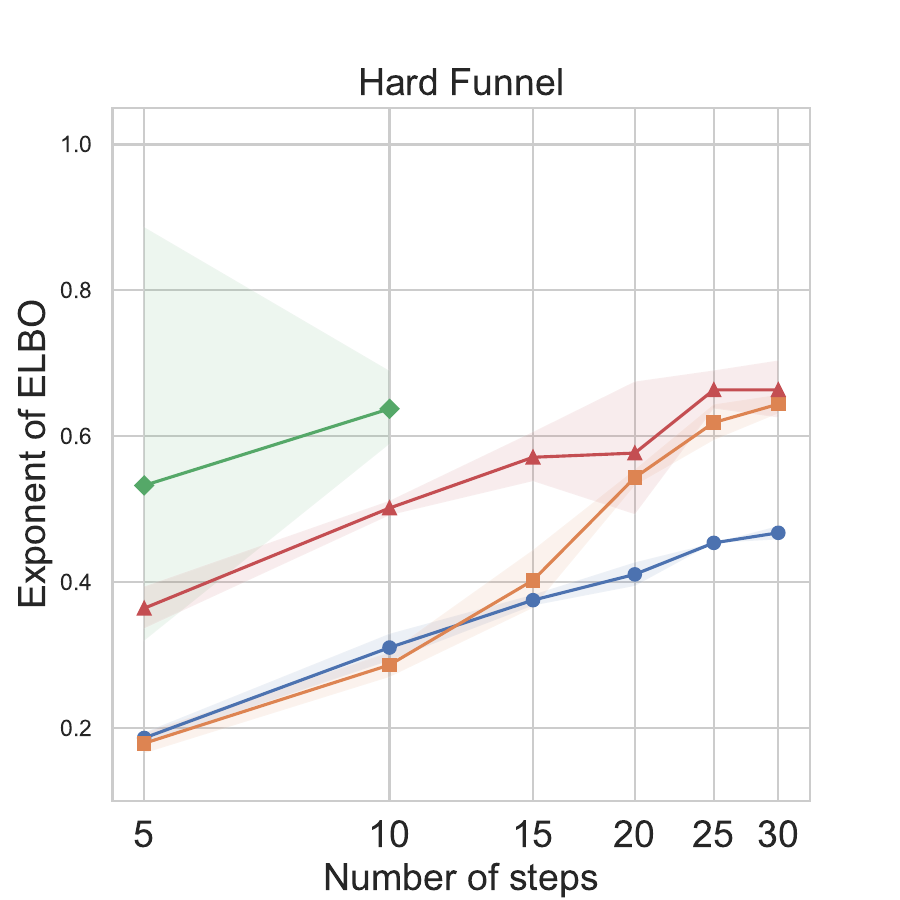
    } 
    &
    \includegraphics[width=0.243\textwidth,trim=0 0 40 49,clip]{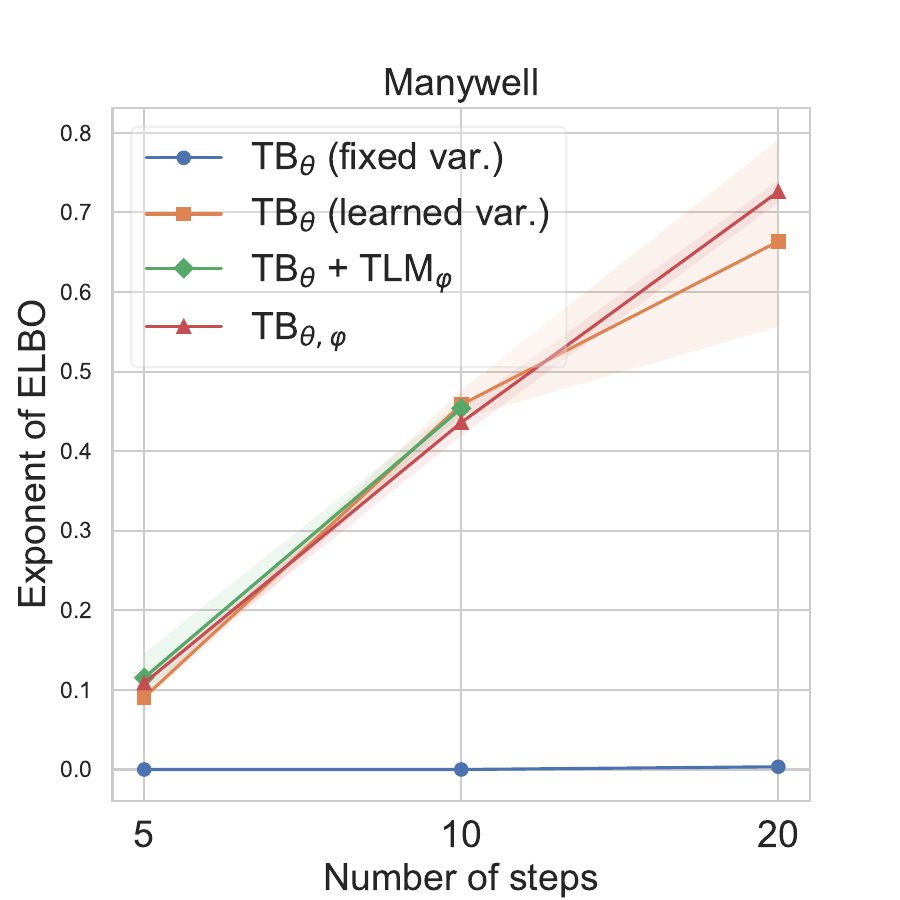} 
    \\
    \end{tabular}
    \caption{
       The same plots as \Cref{fig:main_plots} for \textbf{125GMM}, \textbf{Easy Funnel}, \textbf{Hard Funnel}, \textbf{Distorted Manywell}.
    }
    \label{fig:main_plots_2}
    }
    \vspace*{-1em}
\end{figure}

Representative results are shown in \Cref{fig:main_plots,fig:main_plots_2} and full tables are presented in \Cref{sec:results_tables}.

\looseness=-1
\paragraph{Learned generation variances improve sampling.}

Learnable variance of the generation process significantly improves the performance of the sampler, especially when the number of generation steps is small. As shown in~\Cref{fig:main_plots}, models with learnable generation variance dramatically outperform ones with fixed variance: on some energies, learned-variance samplers with as few as 5 generation steps outperform 20-step samplers with fixed variance. Second, the importance of trainable variance is clearly observable in complex environments or environments with high distortions (\Cref{fig:main_plots_2}).  If the variance is fixed, the sampler is likely to struggle with correctly sampling narrow modes: in particular, the addition of noise of variance $\sigma^2\dt$ on the last step imposes a smoothness constraint on the modelled distribution. Conversely, learning the magnitude of added noise allows to capture the shape of such modes.

\paragraph{Off-policy losses are superior to differentiable simulation.} 

The off-policy TB loss with exploratory behaviour policy is on par with or better than the PIS loss for learning the generation process in all cases, as shown in the right two panels of \Cref{fig:main_plots} and \Cref{sec:results_tables}. This is consistent with findings in past work, such as \cite{richter2024improved,sendera2024improved,kim2025adaptive}. Extending these findings, find that this improvement is maintained when generation variances and the destruction process are learned. Moreover, the TB loss is more memory-efficient than PIS, as it does not require storing the entire computation graph of the  sampling trajectory for backpropagation.

\paragraph{Learning the destruction process is beneficial.}

Learning the destruction process yields an improvement over models with fixed destruction process and learned generation variance on all tasks, although the improvement is often less pronounced when the number of sampling steps is large (\eg, on the Funnel densities in \Cref{fig:main_plots_2}), presumably because the reverse of a fixed destruction process is better modelled by Gaussian transitions when the number of steps is large. See \Cref{fig:funnel_points}, showing that learned variance helps to model the narrow part of the funnel.

\paragraph{TB is preferable to TLM for learning destruction.}

Extending the results of \cite{gritsaev2025optimizing} in discrete cases, we find that training the destruction process with the TLM loss is typically superior to TB when the number of steps is small. However, the TLM loss is unstable and often leads to divergent training when the number of steps is large (see \Cref{sec:results_tables}). We note that, even at optimality, the TLM loss gradient has nonzero variance, while the TB loss gradient is zero for all trajectories at the global optimum, which may explain TB's greater stability.

\begin{figure}[t]
\vspace*{-1em}
\begin{tabular}{@{}p{0.04\linewidth}@{\hspace{0.01\linewidth}}c@{\hspace{0.01\linewidth}}c@{\hspace{0.01\linewidth}}c@{\hspace{0.01\linewidth}}c@{}}
& TB$_\theta$ (fixed var.)
& TB$_\theta$ (learned var.)
& TB$_{\theta,\varphi}$
& Ground truth\\
\cmidrule(lr){2-2}\cmidrule(lr){3-3}\cmidrule(lr){4-4}\cmidrule(lr){5-5}
& $\text{ELBO}=-1.70$ & $\text{ELBO}=-1.73$ & $\text{ELBO}=-0.96$ & $\log Z=0$\\
\rotatebox{90}{\begin{minipage}{4.3125\linewidth}\centering $T=5$\end{minipage}}
&\includegraphics[width=0.23\linewidth,trim=36 20 36 30,clip]{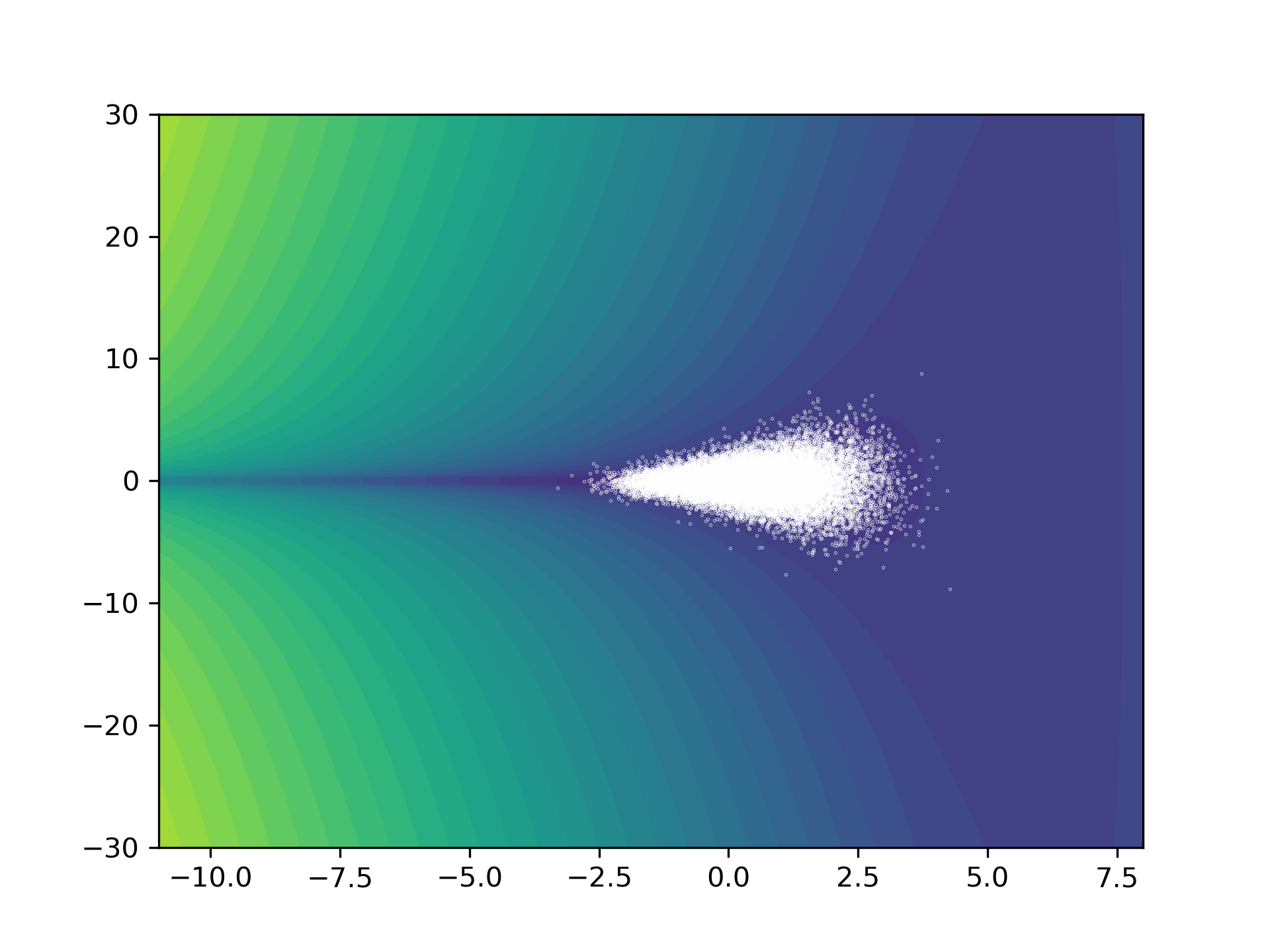}
&\includegraphics[width=0.23\linewidth,trim=36 20 36 30,clip]{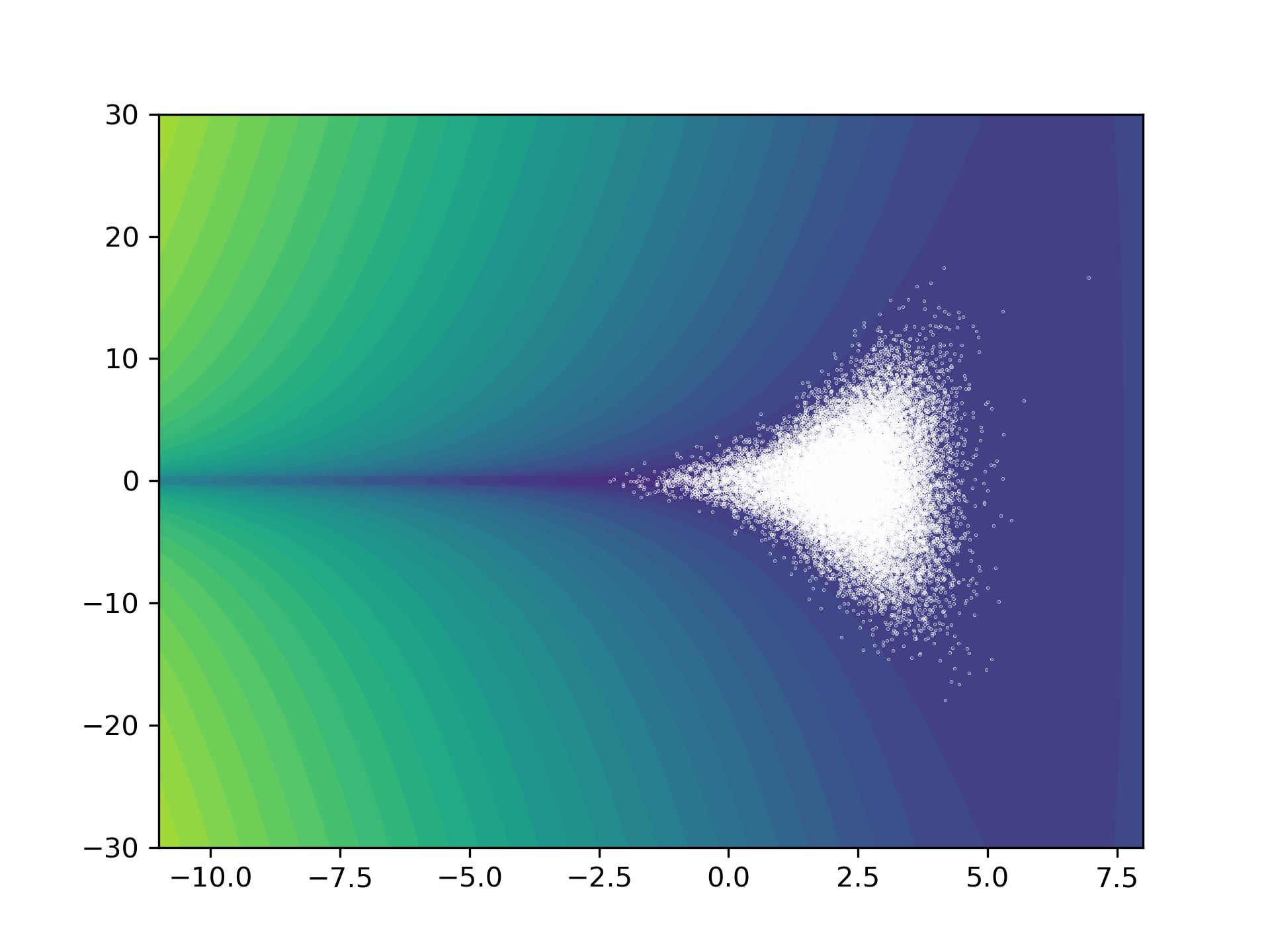}
&\includegraphics[width=0.23\linewidth,trim=36 20 36 30,clip]{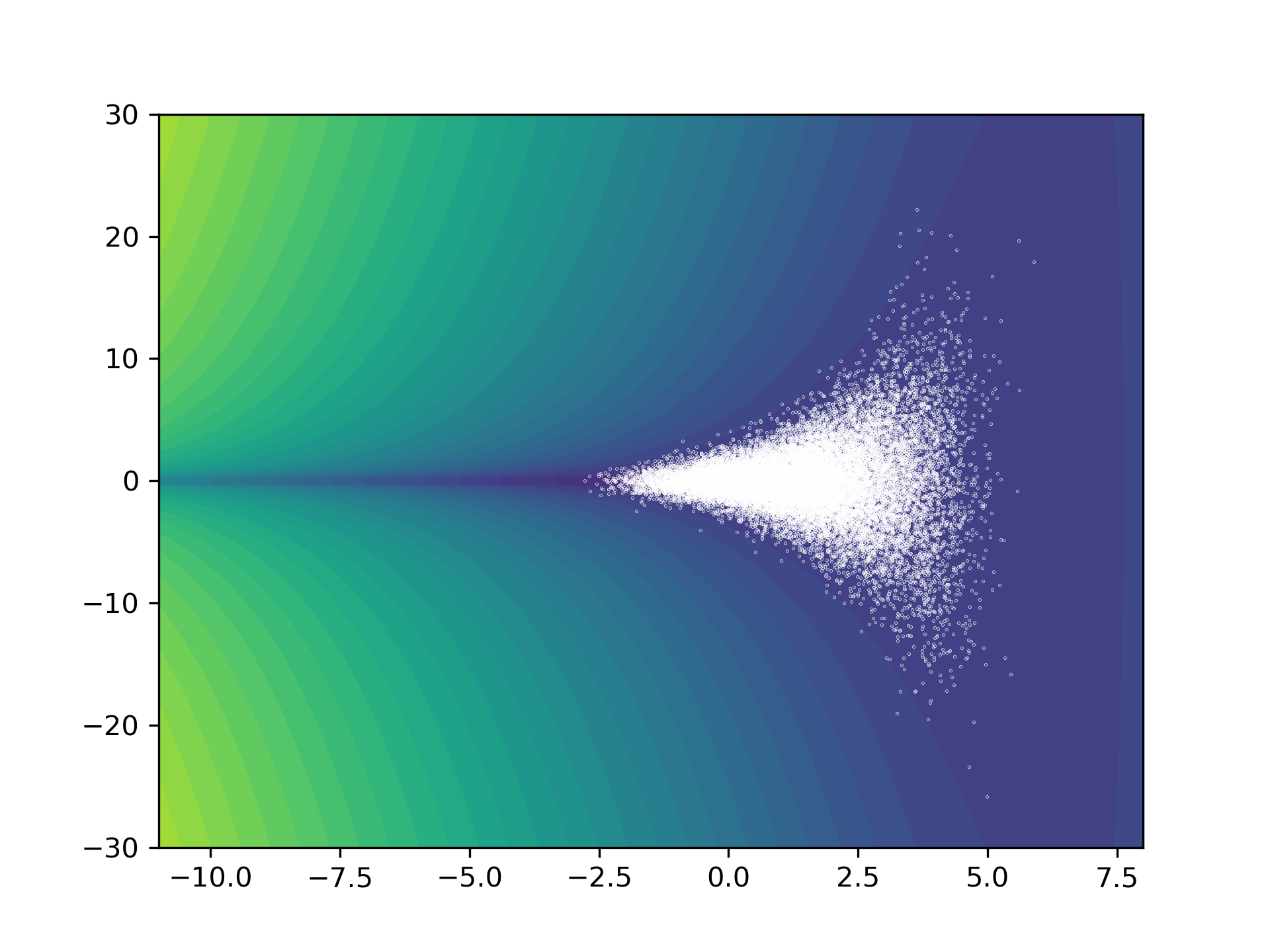}
&\includegraphics[width=0.23\linewidth,trim=36 20 36 30,clip]{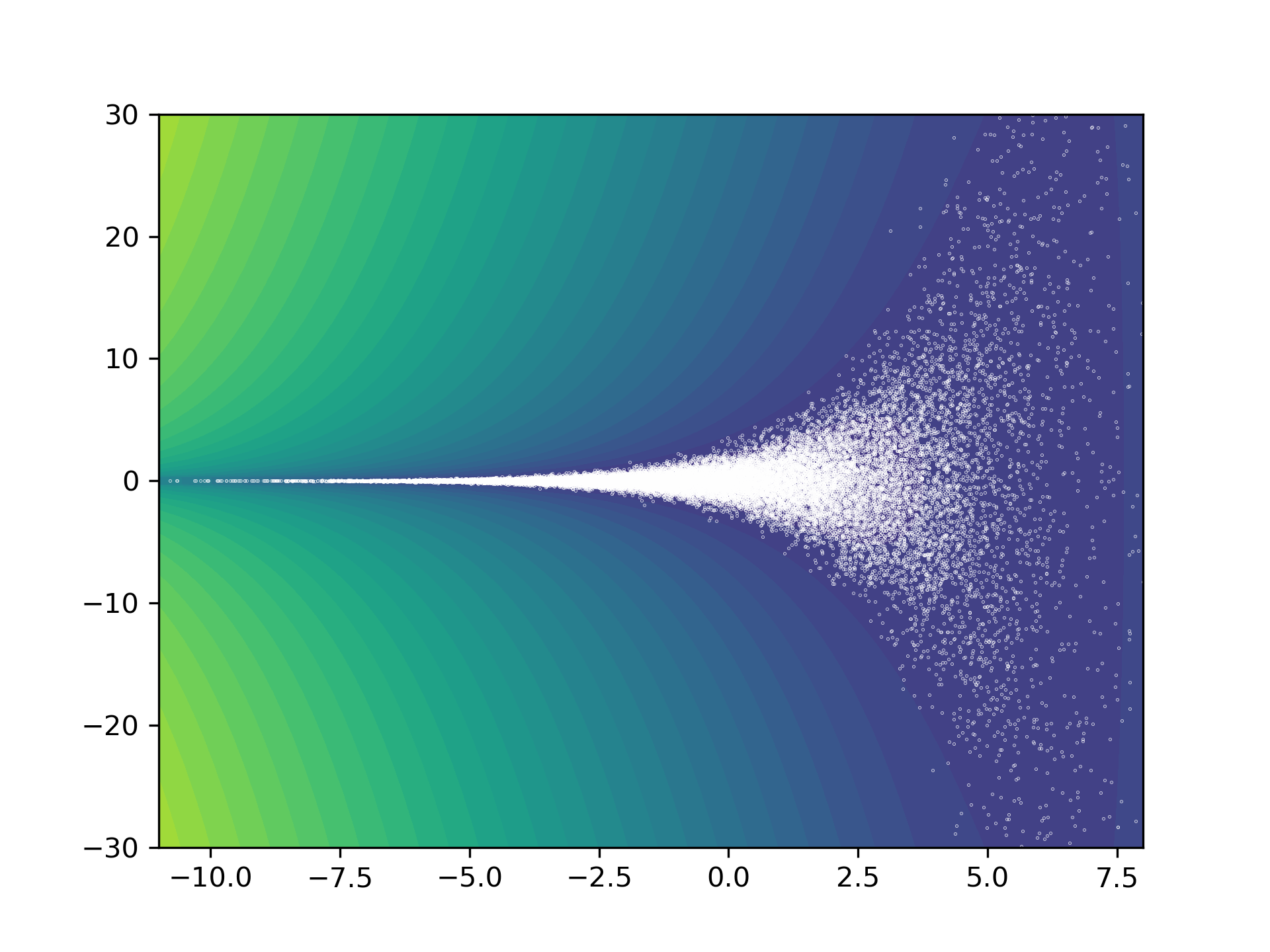}
\\ 
& $\text{ELBO}=-1.22$ & $\text{ELBO}=-1.25$ & $\text{ELBO}=-0.65$ & $\log Z=0$\\
\rotatebox{90}{\begin{minipage}{4.3125\linewidth}\centering$T=10$\end{minipage}}
&\includegraphics[width=0.23\linewidth,trim=36 20 36 30,clip]{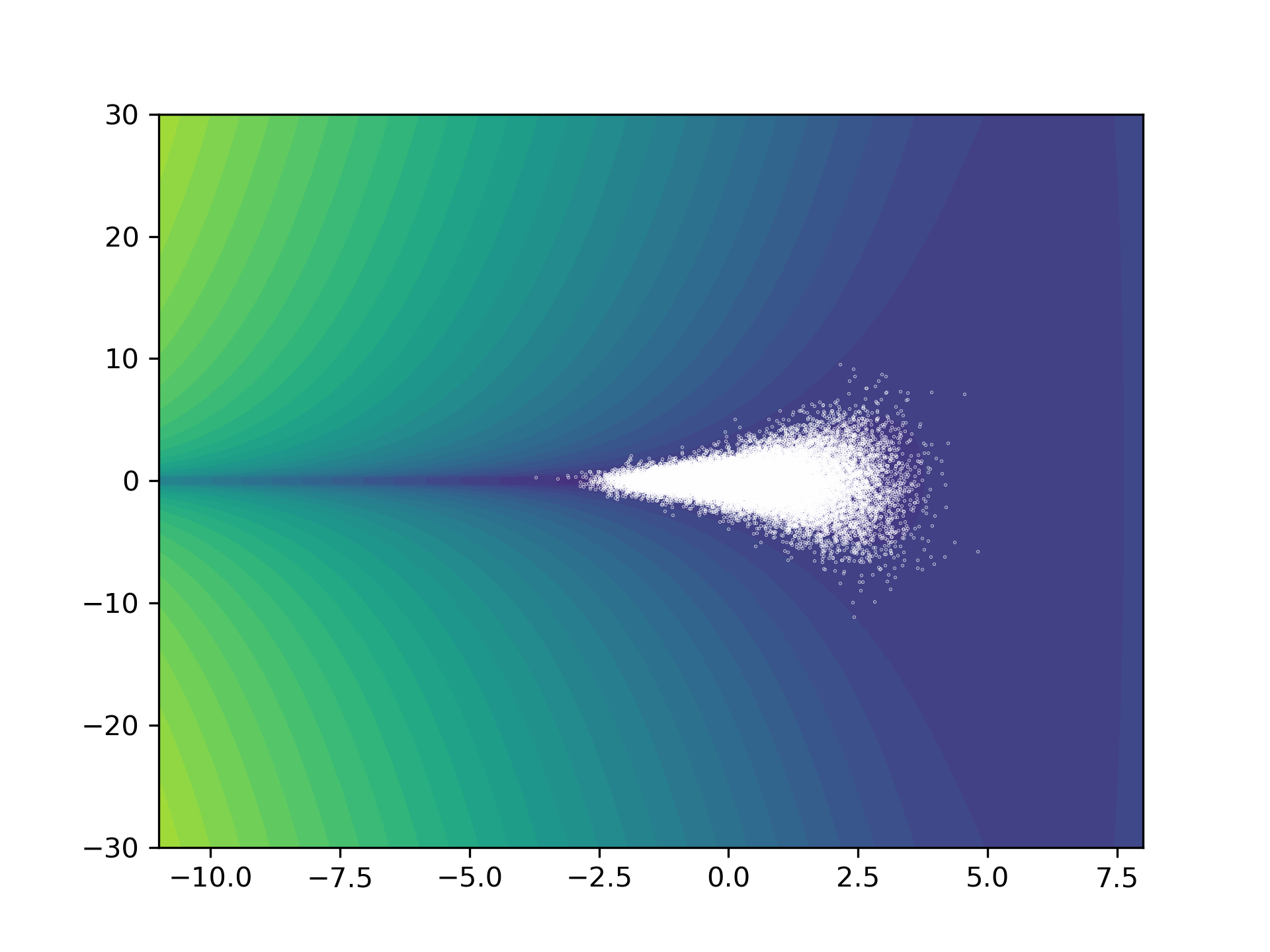}
&\includegraphics[width=0.23\linewidth,trim=36 20 36 30,clip]{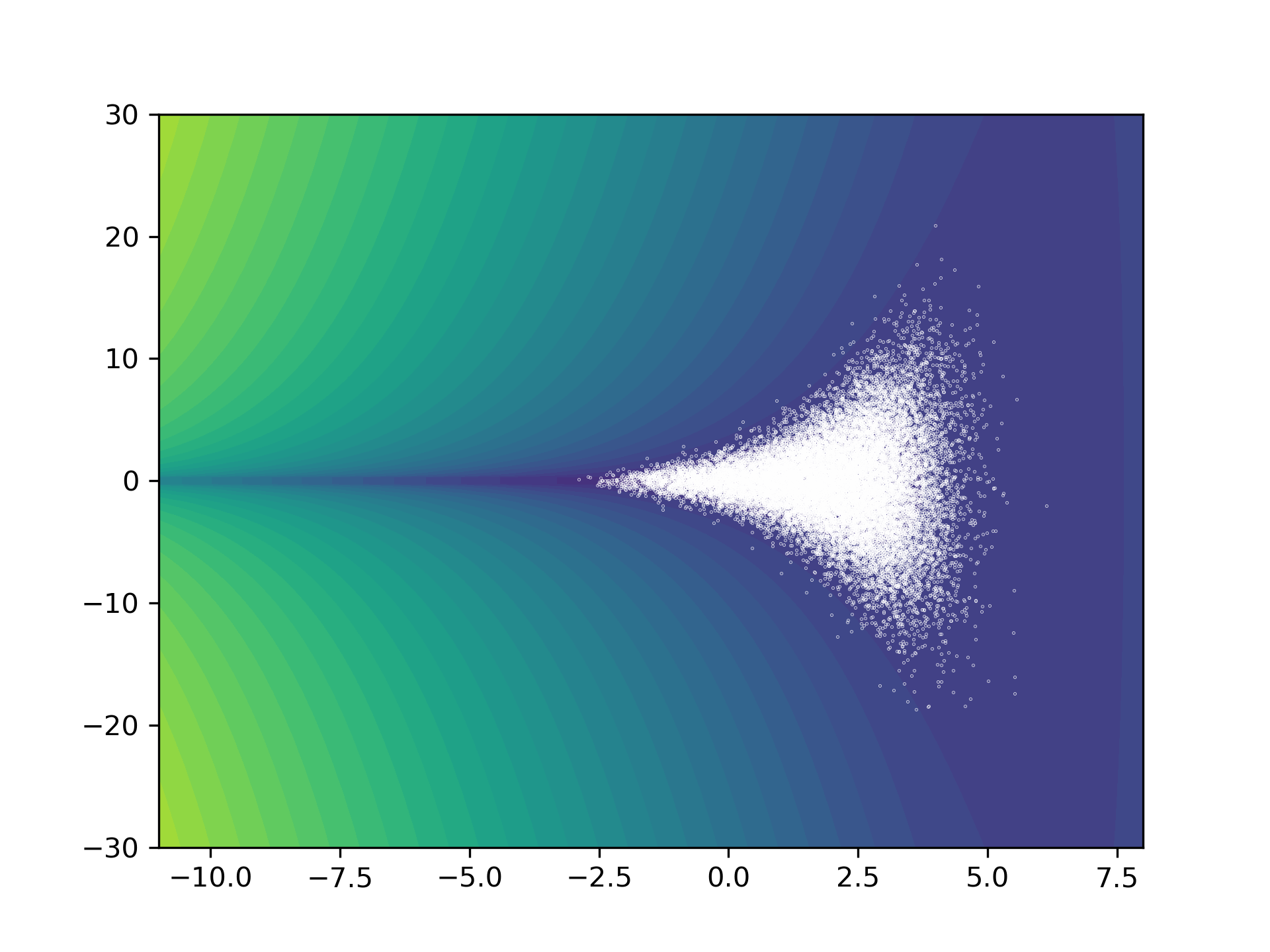}
&\includegraphics[width=0.23\linewidth,trim=36 20 36 30,clip]{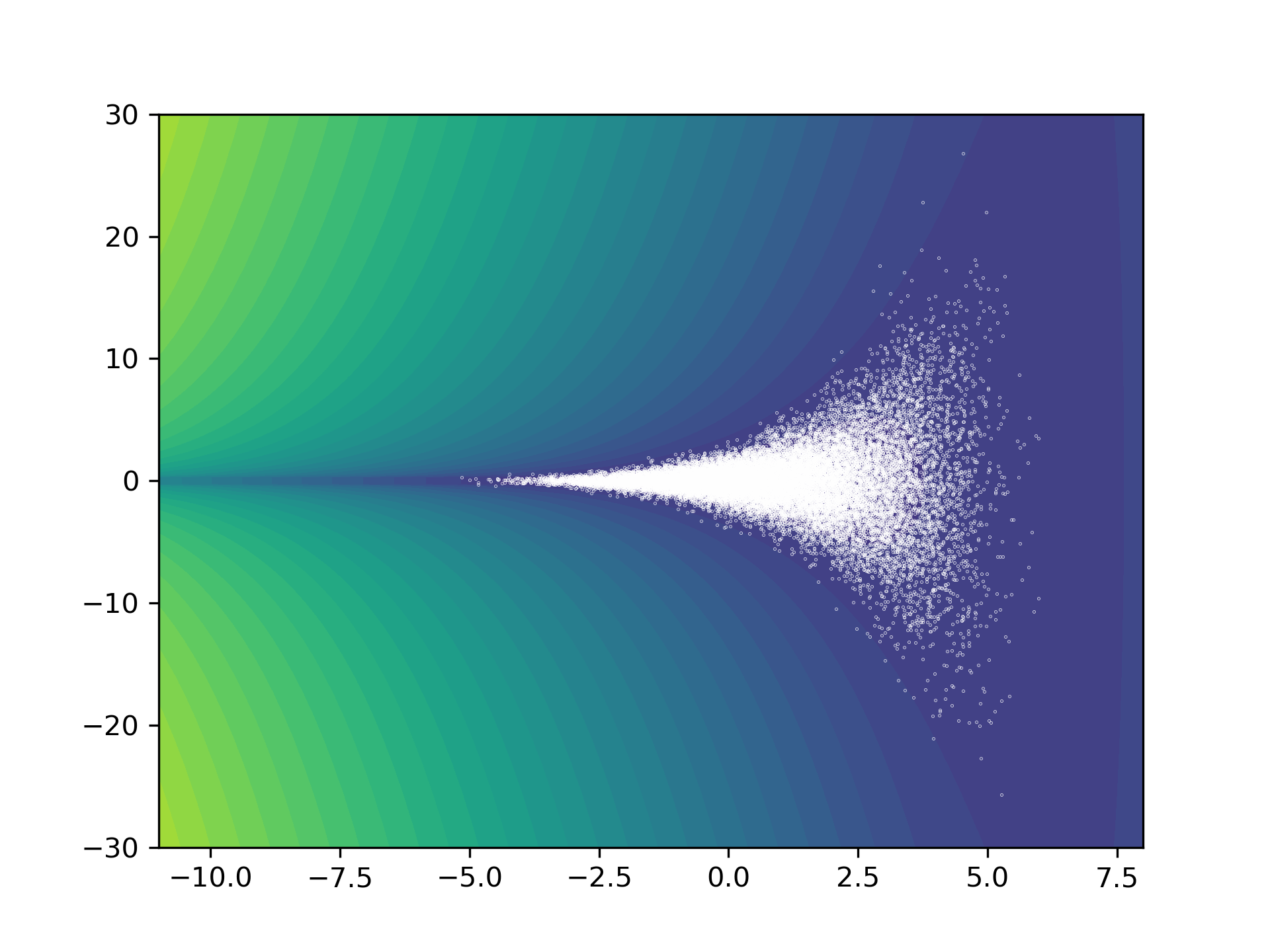}
&\includegraphics[width=0.23\linewidth,trim=36 20 36 30,clip]{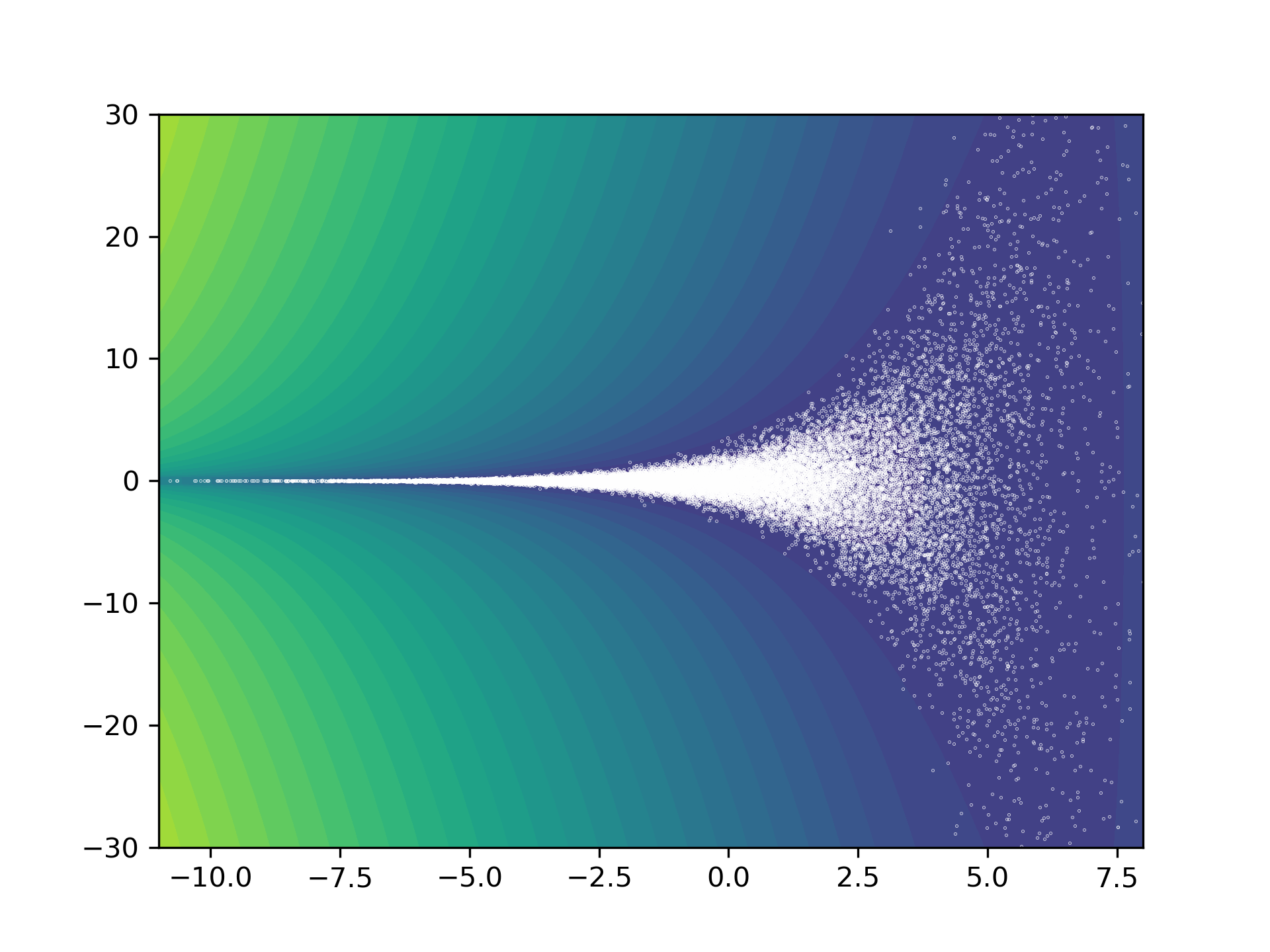}
\end{tabular}
\caption{First two dimensions of target energy and samples from diffusion samplers trained on the Hard Funnel energy. Samplers with fixed destruction process, and especially those with fixed generation process variance, struggle to fit the narrow tails accurately.}
\label{fig:funnel_points}
\end{figure}

\subsection{Ablation study}
\label{sec:ablation_results}

\begin{table}[t]
\vspace*{-1em}
    \centering
    \caption{Different configurations compared by ELBO, EUBO and 2-Wasserstein between generated and ground truth samples  on 40GMM. Mean and std over 5 runs are specified.
    }
    \input{tables/40gmm_ablation_table}

    \label{tab:ablation}
\end{table}

In this subsection we discuss the results of the ablation study to test the design choices outlined in~\Cref{sec:techniques}. The numerical results are presented in~\Cref{tab:ablation}.

\paragraph{Parametrisation, optimisers, learning rates.}

As stated in~\Cref{sec:techniques}, we experiment with using a shared backbone for the generation and destruction processes. We find that using a shared backbone drastically increases the quality of the sampler. Moreover, even though the backbone is shared between two networks, it is optimal to use separate optimisers for the two processes.

We empirically find that the performance of samplers is sensitive to learning rates, and thus they must be carefully tuned for each environment. We use equal learning rates for $\theta$ and $\varphi$ for GMM distributions. However, for more complex environments the destruction policy learning rate must be smaller than that of the generation policy.
For instance, in Hard Funnel, a $10^{3}$ times smaller learning rate for $\varphi$ is optimal, and for Manywell, the optimal ratio is $10^4$ or $10^5$.

\paragraph{Target network and replay buffer.}
The best replay ratio in our setup is $2$, and we use this value for all our experiments. 
Moreover, using target networks increases the stability of the training and the final quality of the sampler. 

\subsection{Scalability: Sampling in GAN latent space for conditional image generation}

\looseness=-1
To validate our main findings in a high-dimensional setting, we consider the setup proposed in~\cite{venkatraman2025outsourced}. Let $g_\psi: \mathbb{R}^{d_{\text {latent}}} \to \mathbb{R}^{d_{\text {data }}}$ be a pretrained GAN generator~\cite{goodfellow2014generative}, and $r(x, y)$ be some positive-valued function operating on data points $x \in \mathbb{R}^{d_{\text {data }}}$ and conditions $y$. The task is to sample latent vectors $z$ from the distribution defined by the energy $\mathcal{E}(z) = -\log p_{\text{prior}}(z) - \beta\log r(g_\psi(z), y)$, where $p_{\text{prior}}$ is $\mathcal{N}(0, I)$. The decoded samples $g_\psi(z)$ then follow the posterior distribution, proportional to the product of the GAN image prior and the tempered constraint $r(x,y)^\beta$. 

\begin{wrapfigure}[19]{r}{0.5\linewidth}
\vspace*{-1.25em}
\captionof{table}{\label{tab:ffhq}FFHQ text-conditional sampling results. All models use $T=5$ discretisation steps.}\vspace*{-0.75em}
\centering
\resizebox{\linewidth}{!}{
\begin{tabular}{@{}lccc}
\toprule   & ELBO $(\uparrow)$ & $\mathbb{E}[\log r(\mathbf{x}, \mathbf{y})]$ $(\uparrow)$ & CLIP Diversity $(\uparrow)$ \\
\midrule
 Prior & $-117.0$ & $-1.17$ & 0.36\\
 \midrule
$\rm TB_{\theta}$ (fixed var.) & $98.8$ & 1.42 & \textbf{0.24} \\
$\rm TB_{\theta, \varphi}$ & \textbf{104.5} & \textbf{1.49} & \textbf{0.24} \\
\bottomrule
\end{tabular}
}\\[0.5em]
TB$_\theta$ (fixed var.) \cite{venkatraman2025outsourced} \\
\includegraphics[width=0.15\linewidth]{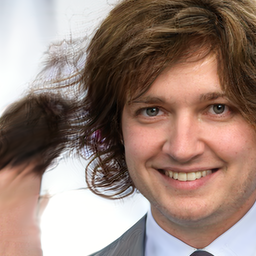}
\includegraphics[width=0.15\linewidth]{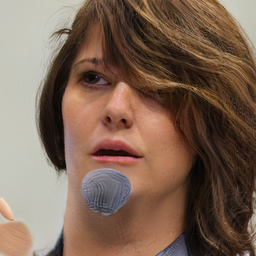}
\includegraphics[width=0.15\linewidth]{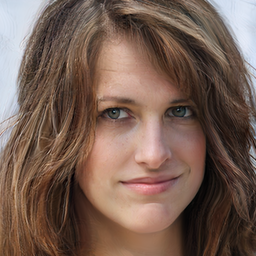}
\includegraphics[width=0.15\linewidth]{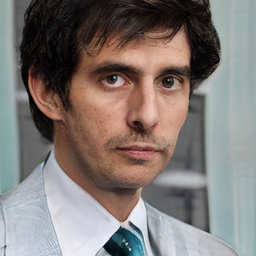}
\includegraphics[width=0.15\linewidth]{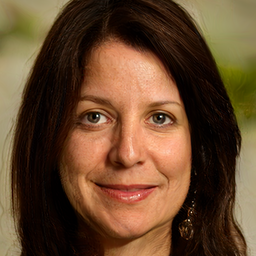}
\includegraphics[width=0.15\linewidth]{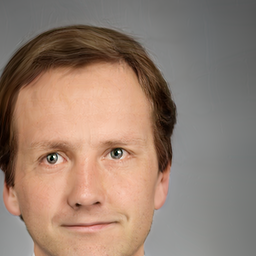}
\\
TB$_{\theta,\varphi}$ (ours) \\
\includegraphics[width=0.15\linewidth]{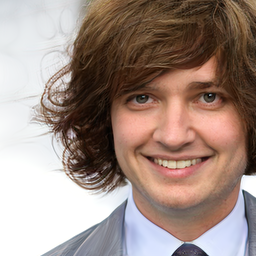}
\includegraphics[width=0.15\linewidth]{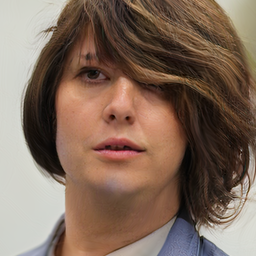}
\includegraphics[width=0.15\linewidth]{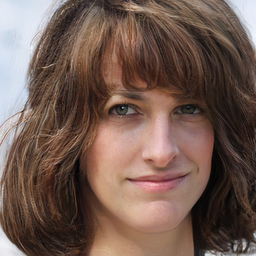}
\includegraphics[width=0.15\linewidth]{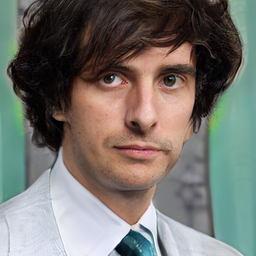}
\includegraphics[width=0.15\linewidth]{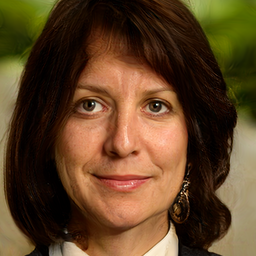}
\includegraphics[width=0.15\linewidth]{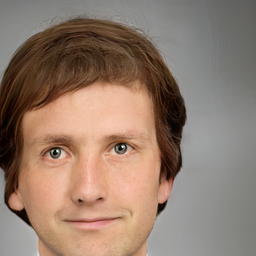}
\captionof{figure}{\label{fig:gan_main}Decoded latents sampled with the same random seeds from outsourced diffusion samplers trained with a StyleGAN3 prior and ImageReward with prompt `A person with medium length hair'.}
\end{wrapfigure}

In our experiment we use StyleGAN3~\cite{karras2021alias} trained on the $256\times256$ FFHQ dataset~\cite{karras2019style} with $d_{\text {latent }} = 512$. We take $y$ to be a text prompt and $\log r(x, y)$ to be the ImageReward score~\cite{xu2023imagereward}; thus our aim is to sample latents that produce faces aligned with the specified text prompt. We use the same value of $\beta=100$ as in~\cite{venkatraman2025outsourced}. Note that the GAN itself is unconditional, and text-conditioning is only achieved via sampling from the distribution defined through ImageReward. An optimal model should trade off between high reward (satisfying the prompt) and high diversity (modelling all modes).

\looseness=-1
We train diffusion samplers with only $T=5$ sampling steps to sample from the specified target distribution in $\mathbb{R}^{d_{\text {latent }}}$, comparing the approach from \cite{venkatraman2025outsourced} ($\rm TB_{\theta}$ with fixed generation variances) to $\rm TB_{\theta, \varphi}$ with learnable destruction process and generation variances. 
In addition to ELBO, we report average ImageReward score and diversity, measured as average cosine distance of CLIP~\cite{radford2021learning} embeddings for 128 generated images. (Note that EUBO and 2-Wasserstein cannot be computed here since we have no access to ground-truth samples.) We train separate models across 7 different prompts and find improvements in ELBO and ImageReward on 5 of them, similar performance on 1, and slight degradation in 1. A representative example of the improvement is shown in \Cref{fig:gan_main}. Average metrics are reported in \Cref{tab:ffhq}. For further details see \Cref{apx:gan_details}.

\section{Conclusion}

In this paper we present the benefits of using learnable variance and learnable destruction process in diffusion samplers. We empirically find that these modifications help to more accurately model complex energy landscapes, especially with few sampling steps. We also contribute to the understanding of training diffusion samplers by studying techniques that improve their stability and convergence speed. We hope that these results inspire the community to scale our findings to other distributions and domains. Another interesting direction for future work would be to rigourously study the optimal parametrisations of the generation and destruction processes -- including non-Gaussian transitions -- and the theoretical limits of sampling with discrete-time learned diffusions.

\section*{Acknowledgements}
This research was supported in part through computational resources of HPC facilities at HSE University~\citep{kostenetskiy2021hpc} and HPC resources from GENCI–IDRIS (Grant 2025-AD011016276). The work of Daniil Tiapkin was supported by the Paris \^Ile-de-France R\'egion in the framework of DIM AI4IDF.

\bibliography{clean}
\bibliographystyle{apalike}

\newpage 

\appendix

\newpage

\section{More on diffusion samplers}
\label{apx:more_theory}

\subsection{On KL divergence between processes with different variances.}
\label{apx:incompar_proc}
{

If two path space measures $\mathbb{P}_1$ and $\mathbb{P}_2$ defined by SDEs have different diffusion coefficients $g_1(t)$ and $g_2(t)$, they are not absolutely continuous with respect to each other, therefore, ${\rm KL}(\mathbb{P}_1 \,\|\, \mathbb{P}_2 ) = \infty$. To see this, notice that the quadratic variation of $\mathbb{P}_i$ on the time interval $[t,t']$ is almost surely $\int_t^{t'}g_i(s)^2\dd s$. Therefore, if $\mathbb{P}_1\ll\mathbb{P}_2$, then the $\int_t^{t'}g_1(s)^2\dd s=\int_t^{t'}g_2(s)^2\dd s$ for every $t<t'$, which implies $g_1=g_2$ if both are continuous.

}

\subsection{On soft RL equivalence.}
\label{apx:soft_rl}
{
In this section, we show that training a diffusion sampler with a finite number of steps and a fixed destruction process is equivalent to solving an entropy-regularized reinforcement learning (RL) problem.

We formalize the RL problem using a finite-horizon Markov Decision Process (MDP) \cite[Chapter 4]{puterman2014markov}, defined as the tuple $\mathcal{M} = (\mathcal{S}, \mathcal{A}, \mathsf{P}, \mathsf{r}, H, s_0)$, where $\mathcal{S}$ and $\mathcal{A}$ are measurable state and action spaces, $\mathsf{P}_h(\dd s' | s, a)$ is a time-inhomogeneous transition kernel, $\mathsf{r}_h(s,a)$ is a time-dependent reward function with terminal reward $\mathsf{r}_H(s)$, $H$ is the planning horizon, and $s_0$ is the initial state. We focus on deterministic MDPs, where the transition kernel is
\[
    \mathsf{P}_h(\dd s' \mid s,a) = \delta_{\mathsf{T}_h(s,a)}(\dd s')\,,
\]
with $\mathsf{T}_h(s,a)$ a deterministic transition map and $\delta_y(\dd x)$ the Dirac measure at $y$. The action space is equipped with a base measure $\dd\mu(a)$ (e.g., the Lebesgue measure).

A time-inhomogeneous policy $\pi = \{\pi_h\}_{h=0}^{H-1}$ is a collection of conditional densities $\pi_h(a|s)$ with respect to $\dd\mu(a)$. The corresponding entropy-regularized (or soft) value function \cite{neu2017unified,geist2019theory} is defined as
\[
    V^{\pi}_{\lambda,0}(s) = \mathbb{E}_{\pi}\left[ \sum_{h=0}^{H-1} \left(\mathsf{r}_h(S_h,A_h) - \lambda\log \pi_h(A_h|S_h)\right) + \mathsf{r}_H(S_H) \,\middle|\, S_0 = s \right]\,,
\]
where $\lambda \geq 0$ is a regularization coefficient, and the expectation is over the trajectory induced by $A_h \sim \pi_h(\cdot | S_h)$ and $S_{h+1} = \mathsf{T}_h(S_h, A_h)$ for $h = 0, \ldots, H-1$.

We call a destruction process $\overleftarrow{p} = (\overleftarrow{p_t})_{t=0,\Delta t,\ldots,1}$ \textit{regular} if $\overleftarrow{p_{\Delta t}}(\dd x_0|x_{\Delta t}) = \delta_0(\dd x_0)$ and $\overleftarrow{p_{t}}(\cdot | x_{t})$ has a full support of $\R^d$ for all $t > \Delta t$. We use $\overleftarrow{p_{t+\Delta t}}(x_t | x_{t+\Delta t})$ as a corresponding density.

\begin{theorem}
    Define $\mathcal{E}(x) \colon \R^d \to \R$ as an energy function, a target density $p_{\rm target}(x) = \exp(-\mathcal{E}(x))/Z $, and a regular destruction process $\overleftarrow{p} = (\overleftarrow{p_t})_{t=0,\Delta t,\ldots,1}$.

    Define a Markov decision process $\mathcal{M}_{\mathrm{DS}}$ with a state space equal to $\mathcal{S} = \R^d$, an action space equal to $\mathcal{A} = \R^d$, deterministic transition kernel defined by the following transition function $\mathsf{T}_h(s,a) = a$, a planning horizon $H=T$, and reward function $\mathsf{r}_h(s,a) = \log \overleftarrow{p}_{h \cdot \Delta t}(a \mid s)$ for $h < H$, with terminal reward $\mathsf{r}_H(x) = -\mathcal{E}(x)$.

    Then, for initial state $s_0 = 0$, the soft value with $\lambda = 1$ in the MDP $\mathcal{M}_{\mathrm{DS}}$ satisfies
    \[
        V^{\pi}_{\lambda=1,0}(s_0) = \log \mathrm{Z} - \mathbb{D}_{\rm KL}(p_0\overrightarrow{p_\pi}\|p_{{\rm target}} \overleftarrow{p})\,,
    \]
    where $\overrightarrow{p_{\pi}}$ is a generation process that corresponds to a policy $\pi$:
    $
        \overrightarrow{p_{\pi,t}}(x_{t+\Delta t} \mid  x_t) = \pi_{t / \Delta t}( x_{t+\Delta t} \mid x_{t} )\,.
    $
    As a result, the policy corresponding to the optimal policy $\pi^\star_h$ in the entropy-regularized MDP with $\lambda = 1$ provides a generation process that samples from the target distribution $p_{\rm target}(x) \propto \exp(-\mathcal{E}(x))$.
\end{theorem}
\begin{proof}
Let us start from the expression \eqref{eq:rev_kl} for the reverse KL divergence between the corresponding generation and destruction processes:
\[
    \mathbb{D}_{\rm KL}(p_0\overrightarrow{p_\pi}\|p_{{\rm target}} \overleftarrow{p}) = \mathbb{E}_{X_{0, \Delta t,\ldots, 1} \sim \overrightarrow{p_\pi}(X_{0,\Delta t \ldots, 1} )} \left[ 
        \log\frac{p_0(X_0)\overrightarrow{p_\pi}(X_{\Delta t \ldots, 1} \mid X_0)}
        {\exp(-\mathcal{E}(X_1))\overleftarrow{p}(X_{0 , \ldots,(T - 1)\dt} \mid X_1)}
    \right]+\log Z.
\]
Next, we use a chain rule for the probability distribution to study the first term in the expansion above:
\begin{align*}
    \mathbb{E}_{X_{0, \Delta t,\ldots, 1} \sim \overrightarrow{p_{\pi}}(X_{0,\Delta t \ldots, 1} )} \left[
    \sum_{i=0}^{T-1} \log \overrightarrow{p_{\pi,i \Delta t}}(X_{(i+1)\Delta t} \mid  X_{i \Delta t}) - \log\overleftarrow{p_{(i+1)\Delta t}}(X_{i\Delta t} \mid X_{(i+1)\Delta t})
    + \mathcal{E}(X_1)
    \right].
\end{align*}
To show the equivalence with the corresponding soft RL definition of the value, we rename variables as follows: $T \mapsto H$, $i \mapsto h$, and $S_{h+1}=A_h= X_{(i+1)\Delta t}$, apply a definition of a reward function and of the generation process induced by policy:
\begin{align*}
    \mathbb{D}_{\rm KL}(p_0\overrightarrow{p_\pi}\|p_{{\rm target}} \overleftarrow{p}) - \log Z = \mathbb{E}_{\pi}\left[ \sum_{h=0}^{H-1} \left(\log \pi_{h}(A_{h}|S_h) - \mathsf{r}_h(S_h, A_h)\right) - \mathsf{r}_{H}(S_H) \right]\,.
\end{align*}
In the right-hand side of the equation we see exactly a negative value function with $\lambda = 1$.
\end{proof}
The same result was obtained by~\cite{tiapkin2024generative} in the GFlowNet framework. In addition, we would like to add that this perspective allows to treat the Trajectory Balance loss as a specific instance of Path Consistency Learning \cite{nachum2017bridging}, as it was shown for GFlowNets in~\cite{deleu2024discrete} (see also~\cite{uehara2024understanding} for a discussion in the context of RL-based fine-tuning of diffusion models).
}

\section{Techniques for stability}

In this section, we describe the implementation details of stability techniques and discuss unsuccessful approaches for destruction process training.   

\label{apx:stability_techniques}
{
\paragraph{Shared backbone.} Our architecture is built upon the one used in~\cite{sendera2024improved}, uses time and state encoders, a shared backbone and 2 different final layers for generation and destruction policy modelling. We use an MLP with GELU \cite{hendrycks2023gaussianerrorlinearunits} activation for the shared backbone. For experiments on Manywell, Distorted Manywell, and GAN, this MLP has 4 layers, and for other energies it has 2.

\paragraph{Separate optimisers.} The optimiser for generation policy updates parameters of time and state encoders, the backbone, and the final layer modelling the generation process. The optimiser for destruction policy updates the parameters of the time and state encoders, the backbone, and the final layer modelling the destruction process. For both optimisers, we use Adam \cite{DBLP:journals/corr/KingmaB14} with standard parameters and weight decay of $10^{-7}$.

\paragraph{Target network.}
Using a target network introduces a specific coefficient $\tau$, which defines the update speed of the target network. On each iteration, the target weights are updated by:
\begin{equation}
    \overline{\theta} = (1-\tau) \overline{\theta} + \tau \theta,
\label{eq:ema_update}
\end{equation}
where $\overline{\theta}$ are the target network weights. If $\tau$ is set to 0, the target network is the current policy network, while higher $\tau$ leads to a slower evolution of the target network. For the generation policy network, the second moment loss in \eqref{eq:second_moment} transforms into:
\begin{equation}
    \!\!\!\!\mathbb{D}_{\tilde p}^{\text{gen.}}(p_0\overrightarrow{p_\theta}\|p_{{\rm target}} \textcolor{red}{\overleftarrow{p_{\overline{\varphi}}}})
    =
    \E_{X_{0,\Delta t, \ldots, 1} \sim \tilde{p}(X_{0,\Delta t \ldots, 1} )} \left[ \log\dfrac{p_0(X_0)\overrightarrow{p_\theta}(X_{\Delta t \ldots, 1} \mid X_0)}{\exp(-\mathcal{E}(X_1))\textcolor{red}{\overleftarrow{p_{\overline{\varphi}}}}(X_{0 , \ldots,(T - 1)\dt} \mid X_1) }+\log\hat{Z}\right]^2,
    \label{eq:tb_for_generation_w_target}
\end{equation}
where $\overline{\varphi}$ is the frozen destruction target network weights, and other variables are the same as in \eqref{eq:second_moment}.

Similarly, the second moment loss for the destruction policy transforms into:
\begin{equation}
    \!\!\!\!\mathbb{D}_{\tilde p}^{\text{destr.}}(p_0\textcolor{red}{\overrightarrow{p_{\overline{\theta}}}}\|p_{{\rm target}} \overleftarrow{p_{\varphi}})
    =
    \E_{X_{0,\Delta t, \ldots, 1} \sim \tilde{p}(X_{0,\Delta t \ldots, 1} )} \left[ \log\dfrac{p_0(X_0)\textcolor{red}{\overrightarrow{p_{\overline{\theta}}}}(X_{\Delta t \ldots, 1} \mid X_0)}{\exp(-\mathcal{E}(X_1))\overleftarrow{p_{\varphi}}{(X_{0 , \ldots,(T - 1)\dt} \mid X_1) } }+\log\hat{Z}\right]^2,
    \label{eq:tb_for_destruction_w_target}
\end{equation}
where $\overline{\theta}$ is the frozen generation target generation network, and other variables are as in \eqref{eq:second_moment}.

\paragraph{Prioritised experience replay.}
We use the implementation of PER~\cite{schaul2016prioritized} from \texttt{torchrl} library \cite{bou2023torchrl}. We set the temperature parameter $\alpha$ to $1.0$ and the importance sampling correction coefficient to $0.1$ (similar parameter values were used in~\cite{tiapkin2024generative}).

\paragraph{Better exploration in off-policy methods.}
We use the existing techniques proposed by \cite{lahlou2023theory,sendera2024improved} to facilitate exploration during training. We use a replay buffer of terminal states updated by Langevin dynamics (as studied by \cite{sendera2024improved}), that is used to sample trajectories for training via the destruction process. We also sample trajectories from the current generation policy, but with increased variance on each step, with the added variance annealed to zero over the first $10\ 000$ iterations (similar to the techniques studied by \cite{malkin2023gflownets,lahlou2023theory}). 

\paragraph{Other considerations.} 
In addition to the techniques discussed in~\Cref{sec:techniques} we also present design choices that we tried in our experiments, but which turned out to be unsuccessful. We share them to offer deeper intuition behind the development of our final methodology:
\begin{itemize}[left=0pt,nosep]
\item\textbf{Different parametrisation.}
We tried to predict destruction variance in the log-scale in the same way as the generation variance \eqref{eq:transtion_density_f_learnable}:
\begin{equation}
    \beta'_{\varphi}(X_t, t) = \exp\left\{C_1 \tanh\left({\text{NN}^{(2)}_\theta(X_t, t)}\right)\right\},
\label{eq:transtion_density_b_learnable_unsuccessful}
\end{equation}
but this parametrisation caused rapid fluctuations in the destruction policy and convergence was dramatically slower than with the parametrisation in \eqref{eq:transtion_density_b_learnable}.
\item\textbf{Linear annealing.}
In experiments on synthetic tasks, we set the constant $C_2$ in \eqref{eq:transtion_density_b_learnable} to $0.9$. We also tried to increase $C_2$ linearly from $0$ to $0.9$ over the duration of training. We initially thought this to be efficient since both generation and destruction processes are less stable in the beginning of training when modes are unexplored. 
\end{itemize}
}

\section{Experiment details}

\subsection{Definition of energies}
\label{apx:energies}

We present contour levels of 2-dimensional Gaussian mixtures in Figure~\ref{fig:gmm_contours}.

\begin{figure}[h!]
    \vspace*{-1em}
    \centering
    \includegraphics[width=1.0\textwidth]{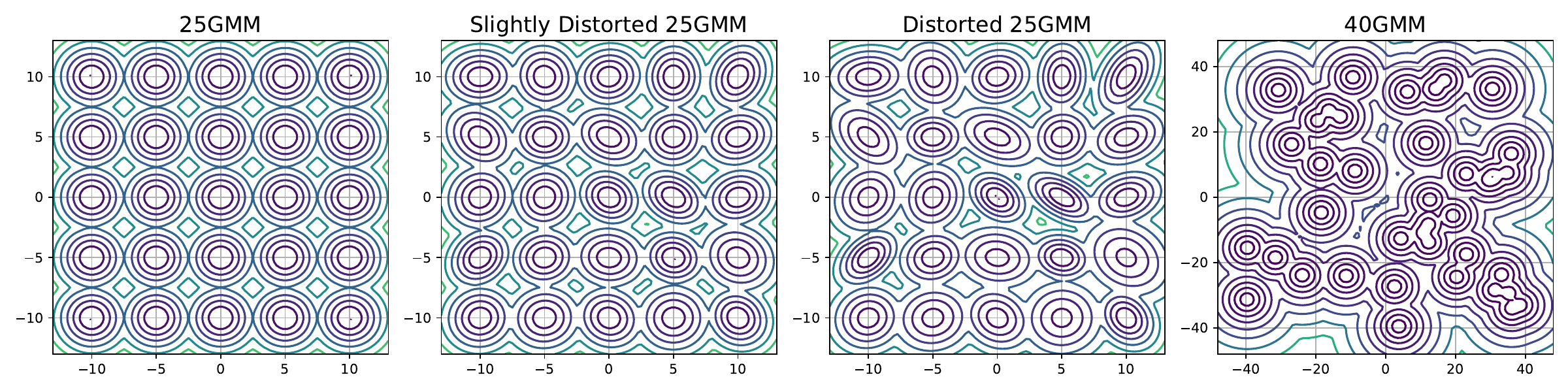}\vspace*{-1em}   
    \caption{Contour levels of 2-dimensional Gaussian mixture densities used in the experiments.}
    \label{fig:gmm_contours} 
    \vspace*{-1em}
\end{figure}

\paragraph{25GMM.} This is a mixture of 25 different Gaussians in a 2-dimensional space. Each component is Gaussian with variance 0.3. The means are arranged on a grid given by the Cartesian product $\{-10, -5, 0, 5, 10\} \times \{-10, -5, 0, 5, 10\}$.

\paragraph{Slightly Distorted 25GMM.} This environment is obtained by a slight modification of 25GMM. The means are located in the same positions, but we distort the initial variance matrices with the following rule:
\begin{equation}    
    C_i = 
    \left(\begin{bmatrix}
        \sqrt{0.3} & 0 \\
        0 & \sqrt{0.3} \\
    \end{bmatrix} + d
    \begin{bmatrix}
        \xi_{i1} & \xi_{i2} \\
        \xi_{i3} & \xi_{i4} \\
    \end{bmatrix}
    \right)^\top \cdot
    \left(\begin{bmatrix}
        \sqrt{0.3} & 0 \\
        0 & \sqrt{0.3} \\
    \end{bmatrix} + d 
    \begin{bmatrix}
        \xi_{i1} & \xi_{i2} \\
        \xi_{i3} & \xi_{i4} \\
    \end{bmatrix}
    \right), \label{eq:distorted_gmm}
\end{equation}
where $C_i$ is the covariance matrix of the $i$-th mode, $\xi_{ij} \sim \mathcal{N}(0, 1)$, and $d$ is set to $0.05$. To achieve a fair comparison, we sample the random variables with a predefined seed $42$, hence all algorithms are compared on the same distribution.

\paragraph{Distorted 25GMM.} The same as Slightly Distorted 25GMM, but with $d$ equal $0.1$.

\paragraph{125GMM.} The same as 25GMM, but in $\mathbb{R}^3$ with means at $\{-10,-5,0,5,10\}^3$.

\paragraph{40GMM.} This distribution is taken from \cite{midgley2022flow}. It consists of 40 equally weighted mixture components with components sampled as:
\begin{align*}
    \pi_{k}(x) &= \mathcal{N}(x;\mu_k, I) \\
    \mu_k &\sim \mathcal{U}(-40, 40).
\end{align*}

\paragraph{Easy/Hard Funnel.} The funnel distribution serves as a classical benchmark in the evaluation of sampling methods. It is defined over a ten-dimensional space, where the first variable, $x_0$, is drawn from a normal distribution with mean 0 and variance 1 (Easy Funnel) or 9 (Hard Funnel), $x_0 \sim \mathcal{N}(0, 1)$. Given $x_0$, the remaining components $x_{1:9}$ follow a multivariate normal distribution with mean vector zero and covariance matrix $\exp(x_0)I$, where $I$ denotes the identity matrix. This conditional relationship is expressed as $x_{1:9} \mid x_0 \sim \mathcal{N}(0, \exp(x_0)I)$.

\paragraph{Manywell.} The distribution is defined over a 32-dimensional space and is constructed as the product of 16 identical 2-dimensional double-well distributions. Each of these two-dimensional components is governed by a potential function, $\mu(x_1, x_2)$, given by $\mu(x_1, x_2) = \exp\left(-x_1^4 + 6x_1^2 + 0.5x_1 - 0.5x_2^2\right)$.

\paragraph{Distorted Manywell.} We modify the Manywell potential function for each 2-dimensional components: 
\begin{equation*}
\mu(x_{2i-1}, x_{2i}) = \exp\left(-a_{i,1} x_{2i-1}^4 + 6 a_{i, 2} x_{2i-1}^2 + 0.5 a_{i,3} x_{2i-1} - 0.5 a_{i,4} x_{2i}^2\right),
\end{equation*}
where $a_{i,j} \sim \mathcal{U}(0.75, 1.25)$. 

\subsection{Metrics}
\label{sec:metrics}

{

The ELBO and EUBO metrics are defined as follows:
\begin{align*}
    \text{ELBO} &= \mathbb{E}_{X_{0, \ldots, 1} \sim p_0(X_0)\overrightarrow{p_\theta}(X_{0,\Delta t \ldots, 1} )} \left[ 
        \log \frac{ \exp(-\mathcal{E}(X_1))\overleftarrow{p_\varphi}(X_{0 , \ldots,(T - 1)\dt} \mid X_1) }
        {p_0(X_0)\overrightarrow{p_\theta}(X_{\Delta t \ldots, 1} \mid X_0)}
    \right]\\
    \text{EUBO} &= \mathbb{E}_{X_{0, \ldots, 1}\sim p_{\rm target}(X_1)p_\varphi(X_{0 , \ldots,(T - 1)\dt} \mid X_1) } \left[ 
        \log \frac{ \exp(-\mathcal{E}(X_1))\overleftarrow{p_\varphi}(X_{0 , \ldots,(T - 1)\dt} \mid X_1) }
        {p_0(X_0)\overrightarrow{p_\theta}(X_{\Delta t \ldots, 1} \mid X_0)}
    \right]
\end{align*}
Both expectations are estimated using $2048$ Monte Carlo samples. For the 2-Wasserstein distance ($W_2^2$), we also use $2048$ ground truth and $2048$ generated samples.

}

\subsection{Synthetic tasks training details}
\label{sec:training_details}
{

Here we describe the chosen hyperparameters for the final results across synthetic tasks in \Cref{sec:synthetic_tasks}. We train all samplers for 25 000 iterations. The diffusion rate $\sigma^2$ is set to 5 for experiments with Gaussian mixtures and to 1 for other energies. The batch size is 512. We apply zero initialisation for the final layers of the neural networks to obtain a uniform output in the beginning of the training. Additionally, we perform clipping on the output of the policy network by $10^{-4}$. We also apply gradient clipping with value $200$. The target neural network update speed $\tau$ is set $0.05$.

We set $C_1$ in \eqref{eq:transtion_density_f_learnable} and $C_2$ in \eqref{eq:transtion_density_b_learnable} to $4.0$ and $0.9$ respectively.

For all experiments, we set lr$_\theta$ (learning rate for generation policy neural network) to $10^{-3}$. The normalisation constant $\log Z$ is trained with learning rate $10^{-1}$. We tune lr$_\varphi$ (learning rate for destruction policy neural network) specifically for each energy and find that optimal lr$_\varphi$ is always less than or equal to lr$_\theta$. For the samplers trained with TB loss, we set lr$_\varphi$ to lr$_\theta$ in experiments with 25GMM and 125GMM, choose the optimal lr$_\varphi$ between \{lr$_\theta$, $10^{-1} \times \text{lr}_\theta$\} for 40GMM, set lr$_\varphi$ to $\text{lr}_\theta$ for Easy Funnel and to $10^{-3} \times \text{lr}_\theta$ for Hard Funnel, we set lr$_\varphi$ to $10^{-5}\times\text{lr}_\theta$ for $T=5$ and to $10^{-4} \times \text{lr}_\theta$ for $T = 10$ and $T=20$ in Manywell and Distorted Manywell. For PIS, we choose optimal lr$_\varphi$ between $\{\text{lr}_\theta, 10^{-1} \times \text{lr}_\theta, 10^{-2} \times \text{lr}_\theta \}$ in experiments with Gaussian mixtures. We use an exponential learning rate schedule, which multiplies learning rate by $\gamma$ after each on-policy gradient step. We set $\gamma$ to $0.99988$ in experiments with 25GMM and 40GMM and to $0.9999$ in other tasks.

For off policy TB, we set the replay ratio to 2. We set the size of the experience replay buffer to 5000 and the size of the buffer used in local search to $600\ 000$. All hyperparameters for local search are taken from \cite{sendera2024improved}. We use exploration factor of $0.3$ for Gaussian mixtures, $0.2$ for Easy Funnel and Hard Funnel, and $0.1$ for Manywell and Distorted Manywell. We note that we use replay buffers and off-policy exploration in all methods and configurations that allow for off-policy training to ensure a fair comparison.

We consider the architecture from \cite{sendera2024improved}, but we stack the time encoding with the state encoding rather than summing them. We set state, time, and hidden dimension sizes to $64$ across all environments with the except for Manywell and Distorted Manywell, where we set these values to $256$.

We use a harmonic discretisation scheme for all mixtures of Gaussians since this time discretisation leads to better sampling quality compared to uniform. It partitions the time space unevenly, making steps large around $t=0$ and smaller closer to $t=1$. The code for harmonic discretisation is presented in~\Cref{app:harmonic}. For other energies, we apply uniform time discretization. 

All models in this section were trained on CPUs. Our implementations are based upon the published code of~\cite{sendera2024improved}.

\definecolor{codegreen}{rgb}{0,0.6,0}
\definecolor{codegray}{rgb}{0.5,0.5,0.5}
\definecolor{codepurple}{rgb}{0.58,0,0.82}
\definecolor{backcolour}{rgb}{0.95,0.95,0.92}

\lstdefinestyle{mystyle}{
    backgroundcolor=\color{backcolour},   
    commentstyle=\color{codegreen},
    keywordstyle=\color{codegreen},
    numberstyle=\tiny\color{codegray},
    stringstyle=\color{codepurple},
    basicstyle=\ttfamily\footnotesize,
    breakatwhitespace=false,         
    breaklines=true,                 
    captionpos=b,                    
    keepspaces=true,                 
    numbers=left,                    
    numbersep=5pt,                  
    showspaces=false,                
    showstringspaces=false,
    showtabs=false,                  
    tabsize=2
}

\lstset{style=mystyle}

\begin{lstlisting}[language=Python, label=app:harmonic]
def harmonic_discretizer(batch_size: int, trajectory_length: int):   
    step_sizes = 1 / arange(1, trajectory_length + 1)
    sum_step_sizes = sum(step_sizes)
    step_proportions = step_sizes / sum_step_sizes
    split_points = cumsum(step_proportions)
    return concatenate([0.0, split_points])
\end{lstlisting}

\subsection{GAN latent sampler training details}
\label{apx:gan_details}
{

We use similar setup and hyperparameters to the ones described in~\Cref{sec:training_details}, with a number of differences. We use smaller values of $C_1 = 1.0$ and $C_2 = 0.1$, which we found to improve the training stability and sampling quality in this task. We also use a smaller batch size of $128$ and larger hidden size of $1024$ for the MLP network (note that~\cite{venkatraman2025outsourced} uses a variant of UNet architecture~\cite{ronneberger2015u}, while we use a variant of the architecture from~\cite{sendera2024improved}, see~\Cref{sec:training_details}). We set $\sigma^2 = 1$, which matches the variance of the prior. We set $\text{lr}_\theta = 10^{-3}$, $\text{lr}_\varphi = 0.2 \times \text{lr}_\theta$, $\gamma = 0.9999$. The exploration factor is set to $0.1$ and the replay ratio is set to $5$. We found it beneficial to set the target neural network update speed $\tau$ to a higher value of $0.15$. We train all samplers for 20 000 iterations.

We note that replay buffers and off-policy exploration are used both for the baseline and for our approach to ensure a fair comparison. We do not use the local search method proposed in~\cite{sendera2024improved} as it requires access to gradients of the target energy function, which would require costly differentiation through the GAN generator and ImageReward. Thus all models considered in this experiment require access only to $\mathcal{E}(z)$.

For GAN latent space sampling experiments we used NVIDIA V100 GPUs. Our implementations are based upon the published code of~\cite{sendera2024improved}, as well as the official implementations of~\cite{karras2021alias, radford2021learning, xu2023imagereward}.

Extending the results from \Cref{tab:ffhq}, \Cref{fig:gan_elbos} depicts metric differences across all prompts utilized in this experiment. \vspace{0.2cm}


\begin{figure}[h!]
  \centering
    \includegraphics[width=1.0\linewidth]{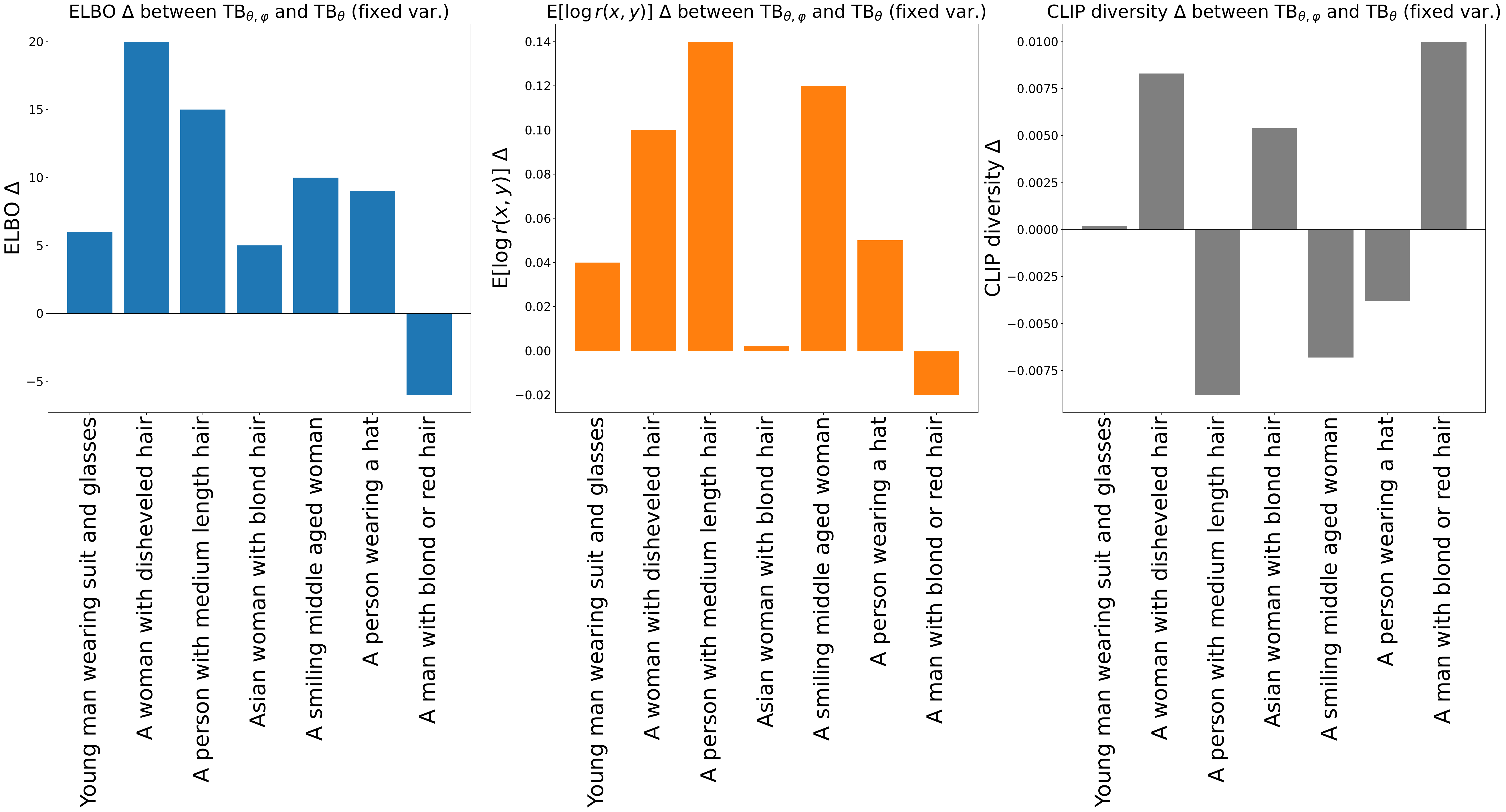} 
  \caption{FFHQ text-conditional latent sampling results across different prompts. The figure depicts the difference in metrics between TB$_{\theta,\varphi}$ (ours) and TB$_\theta$ (fixed var.). \textit{Left:} ELBO, \textit{middle:} average ImageReward, \textit{right:} CLIP diversity. Our method improves reward and ELBO in most prompts. Differences in CLIP diversity are minor (less than 0.01 for all prompts), and on average both methods show the same diversity when rounded up to two decimals.}
\label{fig:gan_elbos}
\end{figure}

}

\newpage

\section{Supplementary tables}
\label{sec:results_tables}
{

See \Cref{tab:25gmm_tables,tab:125gmm_table,tab:40gmm_table,tab:funnel_table,tab:many_well_rrn2_table} on the following pages for full results.

\begin{table}[h!]
    \centering
    \caption{Comparison of 4 algorithms by ELBO, EUBO, and 2-Wasserstein between generated and ground truth samples  with varying number of discretisation steps on 125GMM. Mean and std over 3 runs are specified.
    }
    ELBO  ($\uparrow$)
    \input{tables/125gmm_elbo_table}

    EUBO ($\downarrow$)
    \input{tables/125gmm_eubo_table}

    2-Wasserstein ($\downarrow$)
    \input{tables/125gmm_wass_table}

    \label{tab:125gmm_table}
\end{table}

\begin{table}[]
    \centering
    \caption{Comparison of 4 algorithms by ELBO, EUBO and 2-Wasserstein between generated and ground truth samples with varying number of discretisation steps on 3 variants of 25GMM. Mean and std over 3 runs are specified.
    }
    \scalebox{0.675}{
        \begin{minipage}{\textwidth}
        \centering
            ELBO  ($\uparrow$)\\
            \input{tables/25gm_elbo_table}

            EUBO ($\downarrow$)\\
            \input{tables/25gm_eubo_table}
            2-Wasserstein ($\downarrow$)\\
            \input{tables/25gm_wass_table}
        \end{minipage}
    }
    \label{tab:25gmm_tables}
\end{table}

\begin{table}[]
    \centering
    \caption{Comparison of 4 algorithms by ELBO, EUBO, and 2-Wasserstein between generated and ground truth samples  with varying number of discretisation steps on 40GMM. Mean and std over 3 runs are specified.
    }
    ELBO  ($\uparrow$)
    \input{tables/40gmm_elbo_table}

    EUBO ($\downarrow$)
    \input{tables/40gmm_eubo_table}

    2-Wasserstein ($\downarrow$)
    \input{tables/40gmm_wass_table}

    \label{tab:40gmm_table}
\end{table}

\begin{table}[]
    \centering
    \caption{Comparison of 4 algorithms by ELBO, EUBO, and 2-Wasserstein between generated and ground truth samples  with varying number of discretisation steps on Easy Funnel and Hard Funnel. Mean and std over 3 runs are specified.
    }
    ELBO  ($\uparrow$)
    \input{tables/funnel_elbo_table}

    EUBO ($\downarrow$)
    \input{tables/funnel_eubo_table}

    2-Wasserstein ($\downarrow$)
    \input{tables/funnel_wass_table}

    \label{tab:funnel_table}
\end{table}

\begin{table}[]
    \centering
    \caption{Comparison of 4 algorithms by ELBO gap, EUBO gap, and 2-Wasserstein between generated and ground truth samples  with varying number of discretisation steps on Manywell and Distorted Manywell. Mean and std over 3 runs are specified.
    }
    ELBO gap ($\uparrow$) \\
    \input{tables/many_well_rrn2_elbo_table}

    EUBO gap ($\downarrow$) \\
    \input{tables/many_well_rrn2_eubo_table}
    2-Wasserstein ($\downarrow$) \\
    \input{tables/many_well_rrn2_wass_table}
    \label{tab:many_well_rrn2_table}
\end{table}
}

\end{document}

%% file: tables/40gmm_ablation_table.tex
\resizebox{\linewidth}{!}
{
\begin{tabular}{@{}llccc}
\toprule
Ablation & Setting & ELBO ($\uparrow$) & EUBO ($\downarrow$) & 2-Wasserstein ($\downarrow$) \\
\midrule
\multirow{2}{*}{\it (Optimal configuration)} &
\multirow{2}{*}{
\shortstack[l]{
    TB$_{\theta,\varphi}$, shared backbone, separate optimisers,\\ $\text{lr}_\varphi=\text{lr}_\theta$, $\text{replay ratio}=2$, target network
}
}
& \multirow{2}{*}{$-\textbf{1.89\std{0.10}}$} & \multirow{2}{*}{$\textbf{5.06\std{0.35}}$} & \multirow{2}{*}{$18.71\std{0.63}$} \\
& \\ \midrule
\multirow{2}{*}{Learning only generation} & TB$_\theta$, fixed generation var. & $-4.45\std{0.81}$ & $119.58\std{26.11}$ & $34.21\std{1.76}$ \\ 
& TB$_\theta$, learned generation var. & $-4.36\std{0.04}$ &  $7.02\std{0.62}$ &  $28.29\std{0.49}$ \\
\midrule 
\multirow{2}{*}{Models and optimisers} &  Separate backbones & $-6.03\std{1.15}$ & $26.71\std{2.70}$ &  $28.29\std{0.49}$ \\ 
&  Single optimizer & $-2.08\std{0.15}$ & $13.05\std{4.95}$ & $20.40\std{0.71}$ \\  
\midrule
\multirow{2}{*}{Learning rates} &  $\text{lr}_\varphi = 0.10 \times \text{lr}_{\theta}$ & $-2.06\std{0.20}$ &  $6.55\std{3.48}$ & $\textbf{16.79\std{1.71}}$ \\  
&  $\text{lr}_{\varphi} = 0.01 \times \text{lr}_\theta$ & $-2.69\std{0.94}$ &  $61.43\std{30.20}$ &  $21.77\std{3.68}$ \\
\midrule
\multirow{4}{*}{Replay ratio} &  $\text{Replay ratio}=0$ & $-1.96\std{0.13}$ & $18.59\std{11.72}$ & $20.86\std{1.10}$\\  
&  $\text{Replay ratio}=1$ & $-2.28\std{1.05}$ & $23.53\std{15.92}$ &  $\textbf{16.74\std{1.10}}$ \\  
&  $\text{Replay ratio}=4$ & $-\textbf{1.91\std{0.18}}$ & $6.08\std{1.46}$ & $\textbf{16.74\std{1.10}}$ \\ 
&  $\text{Replay ratio}=8$ & $-3.43\std{1.59}$ &  $21.24\std{21.20}$ & $\textbf{16.79\std{1.71}}$ \\  
\midrule
Target network &
 No target network & $-2.07\std{0.09}$ &  $10.36\std{2.02}$ & $20.13\std{1.10}$
\\\bottomrule
\end{tabular}
}

%% file: tables/125gmm_elbo_table.tex
\resizebox{\linewidth}{!}{
\begin{tabular}{@{}llcccccc}
\toprule
Energy $\downarrow$ & Method $\downarrow$ Steps $\rightarrow$ & $T = 2$ & $T = 3$ & $T = 4$ & $T = 5$ & $T = 10$ & $T = 20$ \\
\midrule
\multirow{8}{4em}{125GMM} 
& TB$_\theta$ (fixed var.) & $-7.36\std{0.09}$ &  $-5.32\std{0.01}$ &  $-4.29\std{0.03}$ &  $-3.64\std{0.06}$ &  $-2.39\std{0.03}$ &  $-1.72\std{0.01}$
\\
& TB$_\theta$ (learned var.) & $-4.92\std{0.16}$ &  $-2.27\std{0.06}$ &  $-1.31\std{0.06}$ &  $-0.88\std{0.08}$ &  $-0.36\std{0.02}$ &  $-0.24\std{0.03}$
\\
& TB$_{\theta}$ + TLM$_{\varphi}$ &  $-3.82\std{1.11}$ & $-1.67\std{0.10}$ & $-1.31\std{0.02}$ & $-0.80\std{0.10}$ & $-0.34\std{0.05}$ & $-0.31\std{0.24}$
\\
& TB$_{\theta,\varphi}$ & $-5.34\std{1.09}$ &  $-2.06\std{0.06}$ &  $-1.31\std{0.13}$ &  $-0.86\std{0.08}$ &  $-0.27\std{0.09}$ &  $-0.06\std{0.00}$ 
\\
\cmidrule(lr){2-8}
& PIS$_\theta$ (fixed var.) &  $-6.94\std{0.32}$   & $-5.34\std{0.63}$   & $-4.75\std{0.73}$   & $-3.68\std{0.13}$   & $-4.85\std{0.13}$   & $-2.83\std{0.12}$  \\
& PIS$_\theta$ (learned var.) &  $-3.18\std{0.59}$  & $-2.75\std{0.26}$  & $-2.09\std{0.04}$  & $-2.05\std{0.04}$  & $-1.87\std{0.01}$  & $-2.06\std{0.08}$ \\
& PIS$_{\theta}$ + TLM$_{\varphi}$ &  $-1.06\std{0.34}$ & $-1.29\std{0.15}$ & $-1.31\std{0.13}$ & $-1.05\std{0.28}$ & $-1.20\std{0.20}$ & $-1.80\std{0.31}$\\
& PIS$_{\theta}$ + VarGrad$_{\varphi}$ &  $-4.83\std{0.00}$ & $-4.83\std{0.00}$ & $-2.25\std{0.25}$ & $-2.04\std{0.01}$ & $-2.13\std{0.47}$ & $-2.16\std{0.54}$\\
\bottomrule
\end{tabular}
}

%% file: tables/125gmm_eubo_table.tex
\resizebox{\linewidth}{!}{
\begin{tabular}{@{}llcccccc}
\toprule
Energy $\downarrow$ & Method $\downarrow$ Steps $\rightarrow$ & $T = 2$ & $T = 3$ & $T = 4$ & $T = 5$ & $T = 10$ & $T = 20$ \\
\midrule
\multirow{8}{4em}{125GMM} 
& TB$_\theta$ (fixed var.) & $10.71\std{0.33}$ &  $9.51\std{0.08}$ &  $8.82\std{0.14}$ &  $8.15\std{0.15}$ &  $6.04\std{0.32}$ &  $4.86\std{0.88}$ 
\\
& TB$_\theta$ (learned var.) & $2.19\std{0.07}$ &  $1.15\std{0.04}$ &  $0.77\std{0.03}$ &  $0.56\std{0.02}$ &  $0.27\std{0.01}$ &  $0.18\std{0.03}$
\\
& TB$_{\theta}$ + TLM$_{\varphi}$ &  $1.58\std{0.36}$ & $1.01\std{0.02}$ & $0.81\std{0.03}$ & $0.52\std{0.06}$ & $0.30\std{0.06}$ & $0.50\std{0.53}$
\\
& TB$_{\theta,\varphi}$ & $2.06\std{0.31}$ &  $1.15\std{0.10}$ &  $0.72\std{0.03}$ &  $0.51\std{0.04}$ &  $0.21\std{0.05}$ &  $0.06\std{0.00}$
\\
\cmidrule(lr){2-8}
& PIS$_\theta$ (fixed var.) &  $11.68\std{0.69}$   & $10.71\std{0.70}$   & $11.68\std{0.61}$   & $11.68\std{0.61}$   & $12.47\std{0.44}$   & $12.83\std{0.20}$  \\
& PIS$_\theta$ (learned var.) &  $11.70\std{1.57}$  & $12.70\std{0.23}$  & $12.81\std{0.15}$  & $12.70\std{0.15}$  & $12.57\std{0.20}$  & $11.61\std{0.57}$ \\
& PIS$_{\theta}$ + TLM$_{\varphi}$ &  $4.98\std{0.86}$ & $4.53\std{0.36}$ & $4.01\std{0.39}$ & $3.30\std{1.56}$ & $3.67\std{0.87}$ & $4.37\std{0.44}$\\
& PIS$_{\theta}$ + VarGrad$_{\varphi}$ &  $8.80\std{2.96}$ & $10.50\std{3.13}$ & $11.89\std{1.12}$ & $11.42\std{1.41}$ & $7.88\std{3.26}$ & $7.60\std{3.75}$\\
\bottomrule
\end{tabular}
}

%% file: tables/125gmm_wass_table.tex
\resizebox{\linewidth}{!}{
\begin{tabular}{@{}llcccccc}
\toprule
Energy $\downarrow$ & Method $\downarrow$ Steps $\rightarrow$ & $T = 2$ & $T = 3$ & $T = 4$ & $T = 5$ & $T = 10$ & $T = 20$ \\
\midrule
\multirow{9}{4em}{125GMM} 
& TB$_\theta$ (fixed var.) & $8.12\std{0.12}$ &  $7.61\std{0.15}$ &  $7.26\std{0.14}$ &  $6.95\std{0.10}$ &  $6.28\std{0.03}$ &  $5.65\std{0.23}$ 
\\
& TB$_\theta$ (learned var.) & $3.46\std{0.09}$ &  $3.09\std{0.13}$ &  $2.95\std{0.20}$ &  $2.74\std{0.04}$ &  $2.43\std{0.05}$ &  $1.93\std{0.05}$
\\
& TB$_{\theta}$ + TLM$_{\varphi}$ &  $1.97\std{0.16}$ & $2.82\std{0.15}$ & $2.60\std{0.14}$ & $2.23\std{0.31}$ & $2.27\std{0.36}$ & $2.37\std{1.04}$
\\
& TB$_{\theta,\varphi}$ & $2.54\std{0.46}$ &  $2.95\std{0.29}$ &  $2.47\std{0.08}$ &  $2.21\std{0.19}$ &  $1.85\std{0.11}$ &  $1.84\std{0.09}$
\\
\cmidrule(lr){2-8}
& PIS$_\theta$ (fixed var.) &  $7.60\std{0.14}$   & $7.01\std{0.36}$   & $7.32\std{0.54}$   & $6.91\std{0.48}$   & $7.19\std{0.31}$   & $7.94\std{0.47}$  \\
& PIS$_\theta$ (learned var.) &  $7.28\std{1.92}$  & $7.32\std{0.59}$  & $5.82\std{0.06}$  & $5.91\std{0.14}$  & $5.73\std{0.07}$  & $6.12\std{0.26}$ \\
& PIS$_{\theta}$ + TLM$_{\varphi}$ &  $3.25\std{0.82}$ & $4.61\std{0.33}$ & $4.96\std{0.05}$ & $4.24\std{0.85}$ & $5.01\std{0.18}$ & $5.34\std{0.20}$\\
& PIS$_{\theta}$ + VarGrad$_{\varphi}$ &  $6.97\std{1.23}$ & $7.83\std{1.06}$ & $5.68\std{0.07}$ & $5.54\std{0.04}$ & $6.22\std{1.09}$ & $6.38\std{1.20}$\\
\cmidrule(lr){2-8}
& Ground Truth & \multicolumn{6}{c}{1.81\std{0.11}}
\\\bottomrule
\end{tabular}
}

%% file: tables/25gm_elbo_table.tex
\resizebox{\linewidth}{!}{
\begin{tabular}{@{}llcccccc}
\toprule
Energy $\downarrow$ & Method $\downarrow$ Steps $\rightarrow$ & $T = 2$ & $T = 3$ & $T = 4$ & $T = 5$ & $T = 10$ & $T = 20$ \\
\midrule
\multirow{8}{4em}{$25$GMM} 
& TB$_\theta$ (fixed var.) & $-4.92\std{0.06}$ &  $-3.53\std{0.05}$ &  $-2.80\std{0.03}$ &  $-2.35\std{0.04}$ &  $-1.53\std{0.01}$ &  $-1.11\std{0.02}$ 
\\
& TB$_\theta$ (learned var.) & $-3.10\std{0.07}$ &  $-1.44\std{0.07}$ &  $-0.80\std{0.07}$ &  $-0.54\std{0.02}$ & $-0.30\std{0.15}$ &  $-0.12\std{0.02}$ 
\\
& TB$_{\theta}$ + TLM$_{\varphi}$ &  $-2.05\std{0.55}$ & $-1.01\std{0.09}$ & $-0.59\std{0.01}$ & $-0.36\std{0.01}$ & $-0.07\std{0.01}$ & $-0.18\std{0.10}$
\\
& TB$_{\theta,\varphi}$ & $-1.83\std{0.04}$ &  $-1.19\std{0.03}$ &  $-0.70\std{0.01}$ &  $-0.42\std{0.03}$ &  $-0.09\std{0.02}$ &  $-0.03\std{0.01}$ 
\\
\cmidrule(lr){2-8}
& PIS$_\theta$ (fixed var.) &  $-4.36\std{0.04}$   & $-3.23\std{0.01}$   & $-2.65\std{0.01}$   & $-2.36\std{0.04}$   & $-1.92\std{0.09}$   & $-1.64\std{0.14}$  
\\
& PIS$_\theta$ (learned var.) &  $-2.14\std{0.81}$  & $-1.45\std{0.02}$  & $-1.37\std{0.01}$  & $-1.31\std{0.01}$  & $-1.20\std{0.01}$  & $-1.34\std{0.13}$ 
\\
& PIS$_{\theta}$ + TLM$_{\varphi}$ &  $-1.13\std{0.19}$ & $-0.76\std{0.04}$ & $-0.82\std{0.16}$ & $-1.13\std{0.01}$ & $-1.01\std{0.17}$ & $-1.81\std{0.30}$
\\
& PIS$_{\theta}$ + VarGrad$_{\varphi}$ &  $-2.37\std{0.00}$ & $-1.75\std{0.38}$ & $-1.29\std{0.01}$ & $-1.18\std{0.10}$ & $-1.19\std{0.01}$ & $-1.13\std{0.01}$
\\
\midrule
\multirow{8}{4em}{Slightly Distorted $25$GMM} 
& TB$_\theta$ (learned var.) & $-3.12\std{0.13}$ &  $-1.47\std{0.03}$ &  $-0.82\std{0.04}$ &  $-0.52\std{0.03}$ &  $-0.22\std{0.01}$ &  $-1.14\std{1.45}$
\\
& TB$_\theta$ (fixed var.) & $-4.80\std{0.07}$ &  $-3.54\std{0.09}$ &  $-2.88\std{0.03}$ &  $-2.44\std{0.04}$ &  $-1.56\std{0.02}$ &  $-1.12\std{0.02}$ 
\\
& TB$_{\theta}$ + TLM$_{\varphi}$ &  $-2.16\std{0.57}$ & $-0.99\std{0.03}$ & $-0.63\std{0.08}$ & $-0.35\std{0.02}$ & $-0.12\std{0.04}$ & $-0.04\std{0.01}$
\\
& TB$_{\theta,\varphi}$ & $-2.49\std{0.34}$ &  $-1.17\std{0.04}$ &  $-0.66\std{0.05}$ &  $-0.42\std{0.03}$ &  $-0.11\std{0.02}$ &  $-0.04\std{0.00}$
\\
\cmidrule(lr){2-8}
& PIS$_\theta$ (fixed var.) &  $-4.15\std{0.05}$   & $-3.39\std{0.04}$   & $-3.14\std{0.45}$   & $-2.51\std{0.09}$   & $-1.97\std{0.09}$   & $-1.80\std{0.07}$  \\
& PIS$_\theta$ (learned var.) &  $-1.21\std{0.27}$  & $-1.47\std{0.02}$  & $-1.40\std{0.01}$  & $-1.31\std{0.02}$  & $-1.29\std{0.05}$  & $-1.51\std{0.19}$ \\
& PIS$_{\theta}$ + TLM$_{\varphi}$ &  $-0.97\std{0.21}$ & $-0.62\std{0.19}$ & $-0.61\std{0.09}$ & $-0.92\std{0.17}$ & $-1.08\std{0.09}$ & $-1.25\std{0.10}$ \\ 
& PIS$_{\theta}$ + VarGrad$_{\varphi}$ &  $-2.38\std{0.69}$ & $-1.48\std{0.15}$ & $-1.30\std{0.01}$ & $-1.22\std{0.10}$ & $-1.04\std{0.11}$ & $-1.19\std{0.05}$
\\
\midrule
\multirow{8}{4em}{Distorted $25$GMM} 
& TB$_\theta$ (fixed var.) & $-4.48\std{0.07}$ &  $-3.51\std{0.05}$ &  $-2.88\std{0.03}$ &  $-2.45\std{0.02}$ &  $-1.60\std{0.03}$ &  $-1.13\std{0.01}$ 
\\
& TB$_\theta$ (learned var.) & $-3.09\std{0.02}$ &  $-1.47\std{0.05}$ &  $-0.84\std{0.05}$ &  $-0.57\std{0.02}$ &  $-0.22\std{0.02}$ &  $-0.12\std{0.02}$
\\
& TB$_{\theta}$ + TLM$_{\varphi}$ &  $-2.20\std{0.31}$ & $-1.18\std{0.05}$ & $-0.63\std{0.04}$ & $-0.34\std{0.02}$ & $-0.12\std{0.02}$ & $-0.04\std{0.01}$
\\
& TB$_{\theta,\varphi}$ & $-2.27\std{0.08}$ &  $-1.37\std{0.04}$ &  $-0.77\std{0.06}$ &  $-0.46\std{0.01}$ &  $-0.12\std{0.01}$ &  $-0.05\std{0.00}$
\\
\cmidrule(lr){2-8}
& PIS$_\theta$ (fixed var.) &  $-4.04\std{0.17}$   & $-3.28\std{0.08}$   & $-2.79\std{0.13}$   & $-2.60\std{0.03}$   & $-2.09\std{0.04}$   & $-2.09\std{0.26}$  
\\
& PIS$_\theta$ (learned var.) &  $-1.33\std{0.17}$  & $-1.71\std{0.16}$  & $-1.52\std{0.15}$  & $-1.78\std{0.16}$  & $-1.65\std{0.23}$  & $-1.91\std{0.06}$ \\
& PIS$_{\theta}$ + TLM$_{\varphi}$ &  $-1.32\std{0.23}$ & $-1.25\std{0.47}$ & $-1.06\std{0.11}$ & $-1.05\std{0.34}$ & $-1.31\std{0.24}$ & $-1.62\std{0.13}$
\\
& PIS$_{\theta}$ + VarGrad$_{\varphi}$ &  $-1.89\std{0.36}$ & $-1.32\std{0.55}$ & $-1.41\std{0.20}$ & $-1.37\std{0.21}$ & $-1.29\std{0.03}$ & $-1.43\std{0.09}$
\\
\bottomrule
\end{tabular}
}

%% file: tables/25gm_eubo_table.tex
\resizebox{\linewidth}{!}{
\begin{tabular}{@{}llcccccc}
\toprule
Energy $\downarrow$ & Method $\downarrow$ Steps $\rightarrow$ & $T = 2$ & $T = 3$ & $T = 4$ & $T = 5$ & $T = 10$ & $T = 20$ \\
\midrule
\multirow{8}{4em}{$25$GMM} 
& TB$_\theta$ (fixed var.) & $7.01\std{0.08}$ &  $5.91\std{0.07}$ &  $5.33\std{0.13}$ &  $5.07\std{0.22}$ &  $4.11\std{0.52}$ &  $2.43\std{0.13}$
\\ 
& TB$_\theta$ (learned var.) & $1.41\std{0.04}$ &  $0.74\std{0.01}$ &  $0.47\std{0.01}$ &  $0.35\std{0.01}$ &  $0.90\std{1.05}$ &  $0.09\std{0.00}$
\\
& TB$_{\theta}$ + TLM$_{\varphi}$ &  $0.99\std{0.14}$ & $0.56\std{0.05}$ & $0.36\std{0.01}$ & $0.24\std{0.01}$ & $0.07\std{0.00}$ & $1.87\std{1.66}$
\\
& TB$_{\theta,\varphi}$ & $0.94\std{0.03}$ &  $0.60\std{0.02}$ &  $0.37\std{0.02}$ &  $0.24\std{0.00}$ &  $0.07\std{0.01}$ &  $0.03\std{0.00}$
\\
\cmidrule(lr){2-8}
& PIS$_\theta$ (fixed var.) &  $7.07\std{0.23}$   & $7.19\std{0.22}$   & $7.31\std{0.54}$   & $7.12\std{0.20}$   & $8.51\std{0.06}$   & $8.51\std{0.06}$  \\
& PIS$_\theta$ (learned var.) &  $8.41\std{0.20}$  & $8.40\std{0.06}$  & $8.42\std{0.04}$  & $8.44\std{0.03}$  & $8.40\std{0.16}$  & $7.02\std{0.05}$ \\
& PIS$_{\theta}$ + TLM$_{\varphi}$ &  $6.03\std{0.61}$ & $5.41\std{0.43}$ & $3.45\std{0.63}$ & $5.50\std{2.13}$ & $4.25\std{1.77}$ & $7.61\std{1.29}$\\
& PIS$_{\theta}$ + VarGrad$_{\varphi}$ &  $8.40\std{0.06}$ & $8.42\std{0.08}$ & $8.43\std{0.08}$ & $6.76\std{1.25}$ & $3.58\std{0.49}$ & $5.71\std{2.01}$
\\
\midrule
\multirow{8}{4em}{Slightly Distorted $25$GMM} 
& TB$_\theta$ (fixed var.) & $7.41\std{0.17}$ &  $6.20\std{0.09}$ &  $5.38\std{0.17}$ &  $4.69\std{0.15}$ &  $3.71\std{0.11}$ &  $2.60\std{0.12}$ 
\\
& TB$_\theta$ (learned var.) & $1.39\std{0.03}$ &  $0.74\std{0.01}$ &  $0.47\std{0.01}$ &  $0.35\std{0.01}$ &  $0.16\std{0.00}$ &  $7.06\std{9.86}$
\\
& TB$_{\theta}$ + TLM$_{\varphi}$ &  $1.06\std{0.32}$ & $0.56\std{0.03}$ & $0.40\std{0.02}$ & $0.26\std{0.02}$ & $0.11\std{0.03}$ & $0.04\std{0.01}$
\\
& TB$_{\theta,\varphi}$ & $1.13\std{0.11}$ &  $0.57\std{0.02}$ &  $0.38\std{0.02}$ &  $0.26\std{0.02}$ &  $0.09\std{0.01}$ &  $0.03\std{0.00}$
\\
\cmidrule(lr){2-8}
& PIS$_\theta$ (fixed var.) &  $7.58\std{0.23}$   & $8.16\std{0.14}$   & $8.33\std{0.21}$   & $8.38\std{0.08}$   & $8.04\std{0.58}$   & $8.37\std{0.13}$  \\
& PIS$_\theta$ (learned var.) &  $8.05\std{0.54}$  & $8.38\std{0.05}$  & $8.38\std{0.05}$  & $8.41\std{0.03}$  & $7.79\std{0.46}$  & $7.52\std{0.22}$ 
\\
& PIS$_{\theta}$ + TLM$_{\varphi}$ &  $5.81\std{1.88}$ & $2.37\std{0.42}$ & $3.90\std{0.12}$ & $3.59\std{1.95}$ & $2.96\std{0.33}$ & $6.01\std{1.02}$
\\
& PIS$_{\theta}$ + VarGrad$_{\varphi}$ &  $8.37\std{0.05}$ & $8.37\std{0.05}$ & $7.69\std{1.00}$ & $5.10\std{1.33}$ & $3.26\std{0.14}$ & $3.79\std{0.93}$
\\
\midrule
\multirow{8}{4em}{Distorted $25$GMM} 
& TB$_\theta$ (fixed var.) & $9.00\std{0.41}$ &  $7.08\std{0.41}$ &  $5.93\std{0.32}$ &  $5.23\std{0.17}$ &  $3.85\std{0.35}$ &  $2.80\std{0.17}$
\\
& TB$_\theta$ (learned var.) & $1.44\std{0.03}$ &  $0.78\std{0.02}$ &  $0.49\std{0.00}$ &  $0.37\std{0.01}$ &  $0.17\std{0.01}$ &  $0.10\std{0.00}$
\\
& TB$_{\theta}$ + TLM$_{\varphi}$ &  $1.13\std{0.15}$ & $0.73\std{0.04}$ & $0.41\std{0.02}$ & $0.27\std{0.01}$ & $0.11\std{0.01}$ & $0.04\std{0.00}$
\\
& TB$_{\theta,\varphi}$ & $1.16\std{0.03}$ &  $0.79\std{0.03}$ &  $0.46\std{0.02}$ &  $0.29\std{0.01}$ &  $0.09\std{0.00}$ &  $0.05\std{0.00}$
\\
\cmidrule(lr){2-8}
& PIS$_\theta$ (fixed var.) &  $8.47\std{0.05}$   & $8.26\std{0.17}$   & $8.47\std{0.05}$   & $8.20\std{0.43}$   & $8.38\std{0.09}$   & $8.47\std{0.05}$  \\
& PIS$_\theta$ (learned var.) &  $8.40\std{0.10}$  & $8.36\std{0.05}$  & $8.36\std{0.05}$  & $8.39\std{0.03}$  & $8.30\std{0.06}$  & $8.43\std{0.07}$ \\
& PIS$_{\theta}$ + TLM$_{\varphi}$ &  $8.36\std{0.05}$ & $8.22\std{0.24}$ & $5.70\std{0.78}$ & $7.14\std{0.91}$ & $5.13\std{2.35}$ & $8.39\std{0.08}$\\
& PIS$_{\theta}$ + VarGrad$_{\varphi}$ &  $8.36\std{0.05}$ & $7.68\std{0.99}$ & $8.10\std{0.37}$ & $6.13\std{1.57}$ & $5.09\std{0.97}$ & $7.57\std{0.69}$\\
\bottomrule
\end{tabular}
}

%% file: tables/25gm_wass_table.tex
\resizebox{\linewidth}{!}{
\begin{tabular}{@{}llcccccc}
\toprule
Energy $\downarrow$ & Method $\downarrow$ Steps $\rightarrow$ & $T = 2$ & $T = 3$ & $T = 4$ & $T = 5$ & $T = 10$ & $T = 20$ \\
\midrule
\multirow{9}{4em}{$25$GMM} 
& TB$_\theta$ (fixed var.) & $6.47\std{0.05}$ &  $6.08\std{0.07}$ &  $5.70\std{0.02}$ &  $5.52\std{0.08}$ &  $5.03\std{0.17}$ &  $4.39\std{0.01}$
\\
& TB$_\theta$ (learned var.) & $2.47\std{0.03}$ &  $2.31\std{0.10}$ &  $2.18\std{0.06}$ &  $2.04\std{0.10}$ &  $2.11\std{0.90}$ &  $1.46\std{0.30}$
\\
& TB$_{\theta}$ + TLM$_{\varphi}$ &  $1.37\std{0.16}$ & $1.45\std{0.03}$ & $1.07\std{0.11}$ & $1.01\std{0.09}$ & $0.91\std{0.03}$ & $2.14\std{1.68}$
\\
& TB$_{\theta,\varphi}$ & $1.77\std{0.17}$ &  $1.76\std{0.27}$ &  $1.07\std{0.15}$ &  $1.13\std{0.11}$ &  $1.05\std{0.20}$ &  $1.17\std{0.24}$
\\
\cmidrule(lr){2-8}
& PIS$_\theta$ (fixed var.) &  $5.71\std{0.10}$   & $5.43\std{0.07}$   & $5.22\std{0.04}$   & $5.24\std{0.08}$   & $5.63\std{0.29}$   & $5.31\std{0.36}$  \\
& PIS$_\theta$ (learned var.) &  $5.90\std{1.46}$  & $4.68\std{0.08}$  & $4.56\std{0.09}$  & $4.46\std{0.06}$  & $4.58\std{0.08}$  & $4.87\std{0.26}$ \\
& PIS$_{\theta}$ + TLM$_{\varphi}$ &  $4.09\std{0.53}$ & $3.47\std{0.26}$ & $3.75\std{0.38}$ & $4.39\std{0.04}$ & $4.37\std{0.11}$ & $6.10\std{1.24}$\\
& PIS$_{\theta}$ + VarGrad$_{\varphi}$ &  $6.98\std{0.03}$ & $5.80\std{1.08}$ & $4.35\std{0.09}$ & $4.17\std{0.27}$ & $4.38\std{0.05}$ & $4.40\std{0.11}$
\\
\cmidrule(lr){2-8}
& Ground Truth & \multicolumn{6}{c}{1.10\std{0.14}}
\\ 
\midrule
\multirow{9}{4em}{Slightly Distorted $25$GMM} 
& TB$_\theta$ (fixed var.) & $6.78\std{0.05}$ &  $6.40\std{0.03}$ &  $5.82\std{0.11}$ &  $5.59\std{0.02}$ &  $5.16\std{0.13}$ &  $4.66\std{0.09}$
\\
& TB$_\theta$ (learned var.) & $2.55\std{0.06}$ &  $2.27\std{0.06}$ &  $2.17\std{0.10}$ &  $1.92\std{0.12}$ &  $1.62\std{0.12}$ &  $6.66\std{7.38}$
\\
& TB$_{\theta}$ + TLM$_{\varphi}$ &  $1.56\std{0.17}$ & $1.54\std{0.22}$ & $1.14\std{0.15}$ & $0.98\std{0.10}$ & $1.15\std{0.19}$ & $0.83\std{0.06}$
\\
& TB$_{\theta,\varphi}$ & $2.06\std{0.27}$ &  $1.52\std{0.11}$ &  $1.28\std{0.20}$ &  $1.04\std{0.13}$ &  $1.10\std{0.20}$ &  $1.04\std{0.13}$
\\
\cmidrule(lr){2-8}
& PIS$_\theta$ (fixed var.) &  $6.40\std{0.39}$   & $5.83\std{0.14}$   & $5.88\std{0.46}$   & $5.73\std{0.25}$   & $5.47\std{0.42}$   & $6.14\std{0.37}$  \\
& PIS$_\theta$ (learned var.) &  $4.36\std{0.92}$  & $4.76\std{0.03}$  & $4.61\std{0.07}$  & $4.56\std{0.14}$  & $4.63\std{0.11}$  & $4.94\std{0.27}$ \\
& PIS$_{\theta}$ + TLM$_{\varphi}$ &  $3.56\std{0.64}$ & $3.15\std{0.52}$ & $3.22\std{0.14}$ & $3.98\std{0.42}$ & $4.37\std{0.19}$ & $4.64\std{0.16}$\\
& PIS$_{\theta}$ + VarGrad$_{\varphi}$ &  $6.17\std{1.16}$ & $4.92\std{0.75}$ & $4.33\std{0.05}$ & $4.31\std{0.19}$ & $4.24\std{0.12}$ & $4.45\std{0.09}$\\
\cmidrule(lr){2-8}
& Ground Truth & \multicolumn{6}{c}{1.06\std{0.14}}
\\
\midrule
\multirow{9}{4em}{Distorted $25$GMM} 
& TB$_\theta$ (fixed var.) & $8.29\std{0.30}$ &  $7.35\std{0.12}$ &  $6.52\std{0.12}$ &  $6.18\std{0.11}$ &  $5.40\std{0.30}$ &  $4.84\std{0.10}$
\\
& TB$_\theta$ (learned var.) & $2.68\std{0.06}$ &  $2.28\std{0.08}$ &  $2.16\std{0.12}$ &  $2.05\std{0.04}$ &  $1.74\std{0.05}$ &  $1.33\std{0.10}$
\\
& TB$_{\theta}$ + TLM$_{\varphi}$ &  $2.03\std{0.23}$ & $1.58\std{0.12}$ & $1.25\std{0.12}$ & $1.00\std{0.20}$ & $0.97\std{0.11}$ & $0.92\std{0.11}$
\\
& TB$_{\theta,\varphi}$ & $2.37\std{0.17}$ &  $1.78\std{0.14}$ &  $1.40\std{0.21}$ &  $1.11\std{0.21}$ &  $1.21\std{0.26}$ &  $1.01\std{0.12}$
\\
\cmidrule(lr){2-8}
& PIS$_\theta$ (fixed var.) &  $6.99\std{0.05}$   & $6.07\std{0.35}$   & $6.11\std{0.49}$   & $5.86\std{0.11}$   & $5.90\std{0.09}$   & $6.68\std{0.45}$  \\
& PIS$_\theta$ (learned var.) &  $4.55\std{0.78}$  & $5.60\std{0.74}$  & $5.14\std{0.73}$  & $6.38\std{0.43}$  & $5.57\std{0.58}$  & $6.61\std{0.27}$ \\
& PIS$_{\theta}$ + TLM$_{\varphi}$ &  $5.14\std{1.30}$ & $4.99\std{1.40}$ & $4.12\std{0.19}$ & $4.43\std{1.15}$ & $5.14\std{0.73}$ & $5.47\std{0.66}$\\
& PIS$_{\theta}$ + VarGrad$_{\varphi}$ &  $5.98\std{1.38}$ & $4.43\std{1.86}$ & $4.87\std{0.92}$ & $4.84\std{0.83}$ & $4.65\std{0.06}$ & $4.95\std{0.36}$
\\
\cmidrule(lr){2-8}
& Ground Truth & \multicolumn{6}{c}{1.06\std{0.14}}
\\\bottomrule
\end{tabular}
}

%% file: tables/40gmm_elbo_table.tex
\resizebox{\linewidth}{!}{
\begin{tabular}{@{}llcccccccc}
\toprule
Energy $\downarrow$ & Method $\downarrow$ Steps $\rightarrow$ & $T = 2$ & $T = 3$ & $T = 4$ & $T = 5$ & $T = 10$ & $T = 15$ & $T = 20$ & $T = 25$ \\
\midrule
\multirow{8}{4em}{40GMM} 
& TB$_\theta$ (fixed var.) & $-4.52\std{1.06}$ &  $-12.95\std{13.36}$ &  $-3.75\std{0.81}$ &  $-3.60\std{1.03}$ &  $-2.40\std{0.06}$ &  $-2.30\std{0.02}$ &  $-2.83\std{0.82}$ &  $-2.29\std{0.12}$
\\
& TB$_\theta$ (learned var.) & $-5.30\std{0.65}$ &  $-3.96\std{0.92}$ &  $-4.96\std{1.15}$ &  $-3.58\std{0.04}$ &  $-2.88\std{0.64}$ &  $-2.26\std{0.13}$ &  $-2.08\std{0.14}$ &  $-1.94\std{0.04}$
\\
& TB$_{\theta}$ + TLM$_{\varphi}$ &  $-3.76\std{0.34}$ & $-2.64\std{0.12}$ & $-2.03\std{0.06}$ & $-1.74\std{0.03}$ & $-1.28\std{0.03}$ & $-1.14\std{0.13}$ & $-1.25\std{0.19}$ & $-1.34\std{0.15}$
\\
& TB$_{\theta,\varphi}$ & $-4.12\std{0.31}$ &  $-2.60\std{0.12}$ &  $-2.44\std{0.17}$ &  $-1.90\std{0.03}$ &  $-1.26\std{0.11}$ &  $-1.15\std{0.10}$ &  $-0.99\std{0.09}$ &  $-1.03\std{0.21}$ 
\\
\cmidrule(lr){2-10}
& PIS$_\theta$ (fixed var.) &  $-4.45\std{0.03}$   & $-3.67\std{0.70}$   & $-3.56\std{0.66}$   & $-3.95\std{0.00}$   & $-3.83\std{0.01}$   & $-3.33\std{0.64}$   & $-3.31\std{0.64}$   & $-3.76\std{0.00}$  \\
& PIS$_\theta$ (learned var.) &  $-3.34\std{0.49}$  & $-2.93\std{0.53}$  & $-2.90\std{0.55}$  & $-2.87\std{0.58}$  & $-2.38\std{0.01}$  & $-3.20\std{0.69}$  & $-2.77\std{0.65}$  & $-2.29\std{0.08}$ \\
& PIS$_{\theta}$ + TLM$_{\varphi}$ &  $-2.53\std{0.11}$ & $-2.62\std{0.00}$ & $-2.22\std{0.08}$ & $-2.34\std{0.02}$ & $-2.29\std{0.05}$ & $-2.33\std{0.01}$ & $-2.32\std{0.01}$ & $-2.32\std{0.01}$\\
& PIS$_{\theta}$ + VarGrad$_{\varphi}$ &  $-2.69\std{0.04}$ & $-2.55\std{0.13}$ & $-2.47\std{0.06}$ & $-2.79\std{0.64}$ & $-2.73\std{0.69}$ & $-2.72\std{0.69}$ & $-2.27\std{0.09}$ & $-2.26\std{0.09}$\\
\bottomrule
\end{tabular}
}

%% file: tables/40gmm_eubo_table.tex
\resizebox{\linewidth}{!}{
\begin{tabular}{@{}llcccccccc}
\toprule
Energy $\downarrow$ & Method $\downarrow$ Steps $\rightarrow$ & $T = 2$ & $T = 3$ & $T = 4$ & $T = 5$ & $T = 10$ & $T = 15$ & $T = 20$ & $T = 25$ \\
\midrule
\multirow{8}{4em}{40GMM} 
& TB$_\theta$ (fixed var.) &  $141.80\std{22.85}$   & $97.40\std{4.91}$   & $251.06\std{95.60}$   & $102.98\std{29.73}$   & $102.84\std{10.99}$   & $2732.36\std{3753.31}$   & $121.65\std{26.31}$   & $98.32\std{22.12}$  \\
& TB$_\theta$ (learned var.) &  $5.33\std{0.15}$  & $7.86\std{2.26}$  & $4.44\std{0.77}$  & $38.06\std{47.29}$  & $5.98\std{2.29}$  & $6.89\std{4.76}$  & $5.29\std{3.08}$  & $37.75\std{46.22}$ \\
& TB$_{\theta}$ + TLM$_{\varphi}$ &  $2.74\std{0.19}$ & $3.38\std{0.18}$ & $3.88\std{0.25}$ & $6.11\std{1.44}$ & $7.25\std{0.88}$ & $5.84\std{0.50}$ & $8.26\std{2.03}$ & $14.95\std{10.41}$ \\
& TB$_{\theta,\varphi}$ &  $3.17\std{0.14}$ & $3.20\std{0.52}$ & $3.18\std{0.29}$ & $5.34\std{0.66}$ & $4.56\std{0.22}$ & $4.64\std{0.17}$ & $4.70\std{0.36}$ & $5.90\std{1.03}$
\\
\cmidrule(lr){2-10}
& PIS$_\theta$ (fixed var.) &  $108.22\std{0.38}$   & $106.93\std{1.45}$   & $106.20\std{2.48}$   & $108.22\std{0.38}$   & $108.22\std{0.38}$   & $108.22\std{0.38}$   & $106.84\std{1.57}$   & $106.39\std{2.21}$  \\
& PIS$_\theta$ (learned var.) &  $106.96\std{0.43}$  & $106.97\std{0.45}$  & $104.85\std{2.55}$  & $106.97\std{0.46}$  & $106.99\std{0.48}$  & $104.08\std{3.64}$  & $104.94\std{3.17}$  & $99.36\std{5.90}$ \\
& PIS$_{\theta}$ + TLM$_{\varphi}$ &  $106.95\std{0.42}$ & $106.95\std{0.44}$ & $106.95\std{0.42}$ & $106.99\std{0.37}$ & $106.98\std{0.40}$ & $107.00\std{0.40}$ & $99.51\std{11.05}$ & $105.41\std{2.71}$\\
& PIS$_{\theta}$ + VarGrad$_{\varphi}$ &  $106.95\std{0.42}$ & $106.95\std{0.44}$ & $106.95\std{0.42}$ & $106.99\std{0.37}$ & $106.98\std{0.40}$ & $102.54\std{5.94}$ & $106.51\std{1.12}$ & $104.24\std{4.27}$
\\
\bottomrule
\end{tabular}
}

%% file: tables/40gmm_wass_table.tex
\resizebox{\linewidth}{!}{
\begin{tabular}{@{}llcccccccc}
\toprule
Energy $\downarrow$ & Method $\downarrow$ Steps $\rightarrow$ & $T = 2$ & $T = 3$ & $T = 4$ & $T = 5$ & $T = 10$ & $T = 15$ & $T = 20$ & $T = 25$ \\
\midrule
\multirow{9}{4em}{40GMM} 
& TB$_\theta$ (fixed var.) &  $38.72\std{5.09}$   & $33.46\std{2.17}$   & $33.93\std{2.24}$   & $31.12\std{2.34}$   & $29.18\std{1.39}$   & $29.93\std{0.16}$   & $27.91\std{2.44}$   & $29.97\std{0.63}$  \\
& TB$_\theta$ (learned var.) &  $18.36\std{0.71}$  & $20.59\std{2.00}$  & $16.34\std{1.75}$  & $21.99\std{6.57}$  & $15.63\std{2.28}$  & $15.14\std{1.33}$  & $14.02\std{0.28}$  & $19.38\std{7.77}$ \\
& TB$_{\theta}$ + TLM$_{\varphi}$ &  $10.02\std{0.23}$ & $12.10\std{2.73}$ & $15.53\std{0.71}$ & $19.84\std{1.53}$ & $18.85\std{0.74}$ & $18.08\std{1.74}$ & $19.64\std{1.49}$ & $20.98\std{1.49}$ \\
& TB$_{\theta,\varphi}$ &  $10.79\std{0.39}$ & $11.47\std{1.38}$ & $11.29\std{1.60}$ & $19.37\std{1.51}$ & $16.45\std{0.79}$ & $16.18\std{0.61}$ & $16.82\std{1.22}$ & $17.83\std{1.77}$
\\
\cmidrule(lr){2-10}
& PIS$_\theta$ (fixed var.) &  $30.44\std{0.07}$   & $30.45\std{0.05}$   & $30.44\std{0.08}$   & $30.43\std{0.07}$   & $30.46\std{0.04}$   & $30.47\std{0.09}$   & $30.46\std{0.05}$   & $30.46\std{0.06}$  \\
& PIS$_\theta$ (learned var.) &  $30.43\std{0.07}$  & $30.37\std{0.14}$  & $30.15\std{0.18}$  & $30.09\std{0.28}$  & $30.09\std{0.21}$  & $30.15\std{0.35}$  & $30.14\std{0.45}$  & $29.73\std{0.60}$ \\
& PIS$_{\theta}$ + TLM$_{\varphi}$ &  $30.22\std{0.35}$ & $30.14\std{0.21}$ & $29.93\std{0.57}$ & $29.92\std{0.07}$ & $29.63\std{0.27}$ & $29.91\std{0.26}$ & $29.85\std{0.17}$ & $29.86\std{0.34}$\\
& PIS$_{\theta}$ + VarGrad$_{\varphi}$ &  $28.33\std{3.00}$ & $30.42\std{0.08}$ & $30.08\std{0.25}$ & $30.06\std{0.50}$ & $30.46\std{0.06}$ & $30.15\std{0.49}$ & $30.22\std{0.24}$ & $30.16\std{0.39}$\\
\cmidrule(lr){2-10}
& Ground Truth & \multicolumn{6}{c}{$3.95\std{0.51}$}
\\\bottomrule
\end{tabular}
}

%% file: tables/funnel_elbo_table.tex
\resizebox{\linewidth}{!}{
\begin{tabular}{@{}llcccccc}
\toprule
Energy $\downarrow$ & Method $\downarrow$ Steps $\rightarrow$ & $T = 5$ & $T = 10$ & $T = 15$ & $T = 20$ & $T = 25$ & $T = 30$ \\
\midrule
\multirow{4}{4em}{Easy Funnel} 
& TB$_\theta$ (fixed var.) & $-0.61\std{0.00}$ &  $-0.34\std{0.02}$ &  $-0.23\std{0.02}$ &  $-0.19\std{0.01}$ &  $-0.15\std{0.00}$ &  $-0.14\std{0.00}$
\\
& TB$_\theta$ (learned var.) & $-0.34\std{0.01}$ &  $-0.17\std{0.01}$ &  $-0.11\std{0.00}$ &  $-0.08\std{0.01}$ &  $-0.07\std{0.01}$ &  $-0.07\std{0.01}$
\\
& TB$_\theta$ + TLM$_\varphi$ & 
$-0.08\std{0.04}$ & $-0.11\std{0.01}$ & diverged & diverged & diverged & diverged 
\\
& TB$_{\theta,\varphi}$ & $-0.15\std{0.01}$ &  $-0.11\std{0.01}$ &  $-0.08\std{0.00}$ &  $-0.06\std{0.01}$ &  $-0.05\std{0.00}$ &  $-0.05\std{0.00}$
\\ 
\midrule
\multirow{4}{4em}{Hard Funnel} 
& TB$_\theta$ (fixed var.) & $-1.68\std{0.02}$ &  $-1.17\std{0.03}$ &  $-0.98\std{0.01}$ &  $-0.89\std{0.02}$ &  $-0.79\std{0.00}$ &  $-0.76\std{0.01}$ 
\\
& TB$_\theta$ (learned var.) & $-1.72\std{0.04}$ &  $-1.25\std{0.03}$ &  $-0.91\std{0.05}$ &  $-0.61\std{0.01}$ &  $-0.48\std{0.02}$ &  $-0.44\std{0.01}$
\\
&  TB$_\theta$ + TLM$_\varphi$ & $-0.63\std{0.26}$ & $-0.45\std{0.04}$ & diverged & diverged & diverged & diverged
\\
& TB$_{\theta,\varphi}$ & $-1.01\std{0.04}$ &  $-0.69\std{0.01}$ &  $-0.56\std{0.03}$ &  $-0.55\std{0.08}$ &  $-0.41\std{0.02}$ &  $-0.41\std{0.03}$
\\
\bottomrule
\end{tabular}
}

%% file: tables/funnel_eubo_table.tex
\resizebox{\linewidth}{!}{
\begin{tabular}{@{}llcccccc}
\toprule
Energy $\downarrow$ & Method $\downarrow$ Steps $\rightarrow$ & $T = 5$ & $T = 10$ & $T = 15$ & $T = 20$ & $T = 25$ & $T = 30$ \\
\midrule
\multirow{4}{4em}{Easy Funnel} 
& TB$_\theta$ (fixed var.) & $0.73\std{0.01}$ &  $0.40\std{0.02}$ &  $0.28\std{0.01}$ &  $0.21\std{0.01}$ &  $0.17\std{0.00}$ &  $0.15\std{0.00}$
\\
& TB$_\theta$ (learned var.) & $0.38\std{0.02}$ &  $0.18\std{0.01}$ &  $0.12\std{0.00}$ &  $0.09\std{0.00}$ &  $0.08\std{0.00}$ &  $0.06\std{0.00}$ 
\\
& TB$_\theta$ + TLM$_\varphi$ & 
$0.08\std{0.03}$ & $0.12\std{0.01}$ & diverged & diverged & diverged & diverged 
\\
& TB$_{\theta,\varphi}$ & $0.17\std{0.03}$ &  $0.12\std{0.01}$ &  $0.08\std{0.00}$ &  $0.07\std{0.01}$ &  $0.05\std{0.00}$ &  $0.05\std{0.00}$
\\ 
\midrule
\multirow{4}{4em}{Hard Funnel} 
& TB$_\theta$ (fixed var.) & $95.36\std{8.41}$ &  $78.77\std{8.35}$ &  $75.40\std{3.35}$ &  $72.83\std{5.15}$ &  $70.99\std{4.26}$ &  $68.93\std{5.01}$
\\
& TB$_\theta$ (learned var.) & $155.51\std{19.60}$ &  $102.97\std{41.86}$ &  $105.67\std{6.26}$ &  $121.76\std{41.45}$ &  $349.17\std{406.05}$ &  $79.37\std{33.30}$ 
\\
&  TB$_\theta$ + TLM$_\varphi$ & $1509.44\std{2037.59}$ & $22.77\std{14.77}$ & diverged & diverged & diverged & diverged
\\
& TB$_{\theta,\varphi}$ & $2800.63\std{1351.92}$ &  $21.48\std{13.09}$ &  $62.95\std{75.95}$ &  $30.35\std{20.62}$ &  $8.60\std{0.97}$ &  $17.29\std{8.63}$
\\\bottomrule
\end{tabular}
}

%% file: tables/funnel_wass_table.tex
\resizebox{\linewidth}{!}{
\begin{tabular}{@{}llcccccc}
\toprule
Energy $\downarrow$ & Method $\downarrow$ Steps $\rightarrow$ & $T = 5$ & $T = 10$ & $T = 15$ & $T = 20$ & $T = 25$ & $T = 30$ \\
\midrule
\multirow{4}{5em}{Easy Funnel} 
& TB$_\theta$ (fixed var.) & $2.45\std{0.04}$ &  $2.43\std{0.03}$ &  $2.45\std{0.04}$ &  $2.49\std{0.01}$ &  $2.47\std{0.02}$ &  $2.48\std{0.03}$
\\
& TB$_\theta$ (learned var.) & $2.49\std{0.02}$ &  $2.50\std{0.03}$ &  $2.49\std{0.03}$ &  $2.49\std{0.02}$ &  $2.50\std{0.02}$ &  $2.53\std{0.01}$
\\
& TB$_\theta$ + TLM$_\varphi$ & 
$2.51\std{0.05}$ & $2.53\std{0.02}$ & diverged & diverged & diverged & diverged 
\\
& TB$_{\theta,\varphi}$ & $2.52\std{0.01}$ &  $2.49\std{0.02}$ &  $2.51\std{0.03}$ &  $2.53\std{0.02}$ &  $2.52\std{0.02}$ &  $2.50\std{0.02}$ \\
 \cmidrule(lr){2-8}
& Ground Truth & \multicolumn{6}{c}{$2.58\std{0.05}$}
\\ 
\midrule
\multirow{4}{5em}{Hard Funnel} 
& TB$_\theta$ (fixed var.) & $22.82\std{0.73}$ &  $22.50\std{0.75}$ &  $22.36\std{0.78}$ &  $22.17\std{0.81}$ &  $22.08\std{0.79}$ &  $21.99\std{0.79}$
\\
& TB$_\theta$ (learned var.) & $22.45\std{0.79}$ &  $21.94\std{0.85}$ &  $21.63\std{0.80}$ &  $21.53\std{0.76}$ &  $21.44\std{0.81}$ &  $21.36\std{0.81}$
\\
&  TB$_\theta$ + TLM$_\varphi$ & $21.32\std{0.90}$ & $21.04\std{0.84}$ & diverged & diverged & diverged & diverged
\\
& TB$_{\theta,\varphi}$ & $21.48\std{0.81}$ &  $20.96\std{0.88}$ &  $21.04\std{0.91}$ &  $20.91\std{0.86}$ &  $21.03\std{0.80}$ &  $21.01\std{0.90}$
\\
 \cmidrule(lr){2-8}
& Ground Truth & \multicolumn{6}{c}{$24.24\std{4.20}$}
\\\bottomrule
\end{tabular}
}

%% file: tables/many_well_rrn2_elbo_table.tex
\resizebox{0.8\linewidth}{!}{
\begin{tabular}{@{}llccc}
\toprule
Energy $\downarrow$ & Method $\downarrow$ Steps $\rightarrow$ & $T = 5$ & $T = 10$ & $T = 20$ \\
\midrule
\multirow{4}{4em}{ManyWell} 
& TB$_{\theta}$ (fixed var.) & $-36.25\std{13.68}$ & $-12.36\std{0.10}$ & $-5.67\std{0.05}$ \\
& TB$_{\theta}$ (learned var.) & $-2.40\std{0.03}$ & $-0.78\std{0.02}$ & $-0.41\std{0.09}$ \\
& TB$_{\theta}$ + TLM$_{\varphi}$ & $-2.16\std{0.12}$ & $-0.79\std{0.01}$ & diverged \\
& TB$_{\theta, \varphi}$ & $-2.22\std{0.05}$ & $-0.83\std{0.02}$ & $-0.32\std{0.01}$
\\ 
\midrule
\multirow{4}{4em}{Distorted ManyWell} 
& TB$_{\theta}$ (fixed var.) & $-77.74\std{6.27}$ & $-25.75\std{9.45}$ & $-11.96\std{0.16}$ \\
& TB$_{\theta}$ (learned var.) & $-8.36\std{0.73}$ & $-5.70\std{1.04}$ & $-3.83\std{0.13}$ \\
& TB$_{\theta}$ + TLM$_{\varphi}$ & $-6.47\std{1.12}$ & $-3.99\std{2.31}$ & diverged \\
& TB$_{\theta, \varphi}$ & $-7.33\std{1.34}$ & $-4.07\std{2.38}$ & $-2.90\std{1.60}$
\\
\bottomrule
\end{tabular}
}

%% file: tables/many_well_rrn2_eubo_table.tex
\resizebox{0.8\linewidth}{!}{
{
\begin{tabular}{@{}llccc}
\toprule
Energy $\downarrow$ & Method $\downarrow$ Steps $\rightarrow$ & $T = 5$ & $T = 10$ & $T = 20$ \\
\midrule
\multirow{4}{4em}{ManyWell} 
& TB$_{\theta}$ (fixed var.) & $12.80\std{3.02}$ & $6.72\std{0.02}$ & $4.06\std{0.04}$ \\
& TB$_{\theta}$ (learned var.) & $1.68\std{0.02}$ & $0.63\std{0.01}$ & $0.36\std{0.09}$ \\
& TB$_{\theta}$ + TLM$_{\varphi}$ & $1.69\std{0.01}$ & $0.72\std{0.02}$ & diverged \\
& TB$_{\theta, \varphi}$ & $1.61\std{0.01}$ & $0.64\std{0.01}$ & $0.31\std{0.01}$ 
\\
\midrule
\multirow{4}{4em}{Distorted ManyWell} 
& TB$_{\theta}$ (fixed var.) & $19.68\std{0.62}$ & $10.46\std{2.30}$ & $6.19\std{0.10} $
\\
& TB$_{\theta}$ (learned var.) & $4.53\std{0.30}$ & $3.16\std{0.37}$ & $2.15\std{0.08}$ \\
& TB$_{\theta}$ + TLM$_{\varphi}$ & $4.09\std{0.49}$ & $2.39\std{1.09}$ & diverged \\
& TB$_{\theta, \varphi}$ & $4.29\std{0.42}$ & $2.38\std{1.19}$ & $1.77\std{0.80}$ 
\\
\bottomrule
\end{tabular}
}
}

%% file: tables/many_well_rrn2_wass_table.tex
\resizebox{0.8\linewidth}{!}{
\begin{tabular}{@{}llccc}
\toprule
Energy $\downarrow$ & Method $\downarrow$ Steps $\rightarrow$ & $T = 5$ & $T = 10$ & $T = 20$ \\
\midrule
\multirow{4}{4em}{ManyWell} 
& TB$_{\theta}$ (fixed var.) & $5.57\std{0.11}$ & $5.40\std{0.02}$ & $5.38\std{0.01}$ \\
& TB$_{\theta}$ (learned var.) & $5.35\std{0.01}$ & $5.38\std{0.01}$ & $5.41\std{0.03}$ \\
& TB$_{\theta}$ + TLM$_{\varphi}$ & $5.33\std{0.03}$ & $5.36\std{0.02}$ & diverged \\
& TB$_{\theta, \varphi}$ & $5.36\std{0.02}$ & $5.39\std{0.01}$ & $5.40\std{0.01}$ \\
\cmidrule(lr){2-5} & Ground Truth &  \multicolumn{3}{c}{$5.42\std{0.02}$}
\\
\midrule
\multirow{4}{4em}{Distorted ManyWell} 
& TB$_{\theta}$ (fixed var.) & $5.87\std{0.03}$ & $5.66\std{0.02}$ & $5.53\std{0.06}$ \\
& TB$_{\theta}$ (learned var.) & $5.53\std{0.00}$ & $5.44\std{0.03}$ & $5.40\std{0.01}$ \\
& TB$_{\theta}$ + TLM$_{\varphi}$ & $5.50\std{0.02}$ & $5.42\std{0.09}$  & diverged \\
& TB$_{\theta, \varphi}$ & $5.54\std{0.01}$ & $5.41\std{0.08}$ & $5.37\std{0.06}$ \\
\cmidrule(lr){2-5} & Ground Truth &  \multicolumn{3}{c}{$5.30\std{0.01}$} \\
\bottomrule
\end{tabular}
}